\documentclass{article}

\PassOptionsToPackage{numbers, compress}{natbib}

\usepackage[preprint]{neurips_2024}

\usepackage[utf8]{inputenc} 
\usepackage[T1]{fontenc}    
\usepackage{hyperref}       
\usepackage{url}            
\usepackage{booktabs}       
\usepackage{amsfonts}       
\usepackage{nicefrac}       
\usepackage{microtype}      
\usepackage{xcolor}         

\usepackage{caption}
\usepackage{subcaption}
\usepackage{amsmath}
\usepackage{graphicx}
\usepackage{algorithm}
\usepackage{algorithmic}
\usepackage{amsthm}
\usepackage{amssymb}
\usepackage{thm-restate}
\usepackage{pifont} 
\usepackage{makecell} 

\newtheorem{theorem}{Theorem}[section]

\newtheorem{lemma}[theorem]{Lemma}

\newtheorem{definition}[theorem]{Definition}
\newtheorem{assumption}[theorem]{Assumption}

\title{Conflict-Averse Gradient Aggregation for \\ Constrained Multi-Objective Reinforcement Learning}

\author{%
Dohyeong Kim$^1$, Mineui Hong$^1$, Jeongho Park$^1$, and Songhwai Oh$^1$\thanks{Corresponding author: songhwai@snu.ac.kr.} \\
$^{1}$Dep. of Electrical and Computer Engineering, Seoul National University \\
}

\begin{document}

\maketitle

\begin{abstract}
In many real-world applications, a reinforcement learning (RL) agent should consider multiple objectives and adhere to safety guidelines.
To address these considerations, we propose a constrained multi-objective RL algorithm named \textit{\textbf{Co}nstrained \textbf{M}ulti-\textbf{O}bjective \textbf{G}radient \textbf{A}ggregator (CoMOGA)}. 
In the field of multi-objective optimization, managing conflicts between the gradients of the multiple objectives is crucial to prevent policies from converging to local optima.
It is also essential to efficiently handle safety constraints for stable training and constraint satisfaction.
We address these challenges straightforwardly by treating the maximization of multiple objectives as a constrained optimization problem (COP), where the constraints are defined to improve the original objectives.
Existing safety constraints are then integrated into the COP, and the policy is updated using a linear approximation, which ensures the avoidance of gradient conflicts.
Despite its simplicity, CoMOGA guarantees optimal convergence in tabular settings.
Through various experiments, we have confirmed that preventing gradient conflicts is critical, and the proposed method achieves constraint satisfaction across all tasks.
\end{abstract}

\section{Introduction}





Many real-world reinforcement learning (RL) applications involve multiple objectives and have to consider safety simultaneously.
For instance, locomotion tasks for legged robots focus on maximizing goal-reaching success rate and energy efficiency while avoiding collisions with the environment, as illustrated in Figure \ref{fig: overview}.
To effectively tackle these considerations, constrained multi-objective RL (CMORL) has been introduced \citep{huang22lp3}.
CMORL aims to find a set of \textit{constrained-Pareto (CP)} optimal policies \citep{hayes2022practical}, which are not dominated by any others while satisfying safety constraints, rather than finding a single optimal policy.
This allows users to choose and utilize their preferred policy from the set without additional training.
In order to express a set of CP optimal policies, it is common to use a concept called \textit{preference} \citep{yang2019envelop, kyriakis2022pareto}, which encodes the relative importance of the objectives. 
Then, a set of policies can be represented by a preference-conditioned policy, a mapping function from a preference space to a policy space.

In most multi-objective RL (MORL) algorithms \citep{jean2023mgda, yang2019envelop, basaklar2023pdmorl}, the preference-conditioned policy is trained by maximizing a scalarized reward, computed as the dot product of a given preference and rewards of the multiple objectives.
These approaches, called linear scalarization (LS), offer the benefits of straightforward implementation and ensure Pareto-optimal convergence in tabular settings \citep{lu2023multiobjective}.
However, when applied to deep RL with function approximators, LS tends to converge to local optima due to the nonlinearity of the objective functions, as shown in a toy example in Figure \ref{fig: toy example}.
In the field of multi-task learning (MTL) \citep{yu2020pcgrad, liu2021cagrad, navon2022nashmtl}, this issue has been addressed by avoiding conflicts between the gradients of multiple objective functions. 
We hypothesize that similar improvements can be achieved in CMORL by avoiding gradient conflicts.

It is also crucial to handle constraints as well as multiple objectives in CMORL.
A straightforward approach is treating the constraints as other objectives and concurrently adjusting the preferences corresponding to those objectives to maintain them below thresholds, as done in \citep{huang22lp3}.
This approach exactly aligns with the Lagrangian method, which solves the Lagrange dual problem by updating policy and multipliers alternatively.
However, the Lagrangian method can make training unstable due to the concurrent update of policy and multipliers \citep{stooke2020pid, kim2022offtrc}.
As the updates of the policy and multipliers influence each other, any misstep in either can make the training process diverge.
Thus, it is critical to manage constraints without introducing additional optimization variables.


In order to avoid gradient conflicts and efficiently handle safety constraints, we introduce a novel but straightforward algorithm named \textit{\textbf{Co}nstrained \textbf{M}ulti-\textbf{O}bjective \textbf{G}radient \textbf{A}ggregator (CoMOGA)}. 
CoMOGA treats a multi-objective maximization problem as a constrained optimization problem (COP), where constraints are designed to enhance the original objective functions in proportion to their respective preference values. 
Existing safety constraints are then incorporated into the constraint set of the COP.
By explicitly preventing the given objectives from decreasing, CoMOGA successfully avoids gradient conflicts.
The policy gradient is then calculated by solving this transformed problem using linear approximation within a local region, thereby eliminating the need for extra variables to handle safety constraints. 
While the proposed method is intuitive and straightforward, we have also demonstrated that it converges to a CP optimal policy in a tabular setting.

The proposed method has been evaluated across diverse environments, including multi-objective tasks with and without constraints.
The experimental results support our hypothesis that avoiding gradient conflicts is effective for preventing convergence to local optima. 
This is particularly essential in environments with unstable dynamics, such as bipedal and humanoid locomotion tasks, since they are prone to get stuck in suboptimal behaviors, such as frequent falling.
Furthermore, the absence of additional variables enables stable training with constraint satisfaction across all tasks.
\begin{figure}[t]
    \centering
    \begin{minipage}[c]{0.39\textwidth}
        \centering
        \vspace{10pt}
        \includegraphics[width=\textwidth]{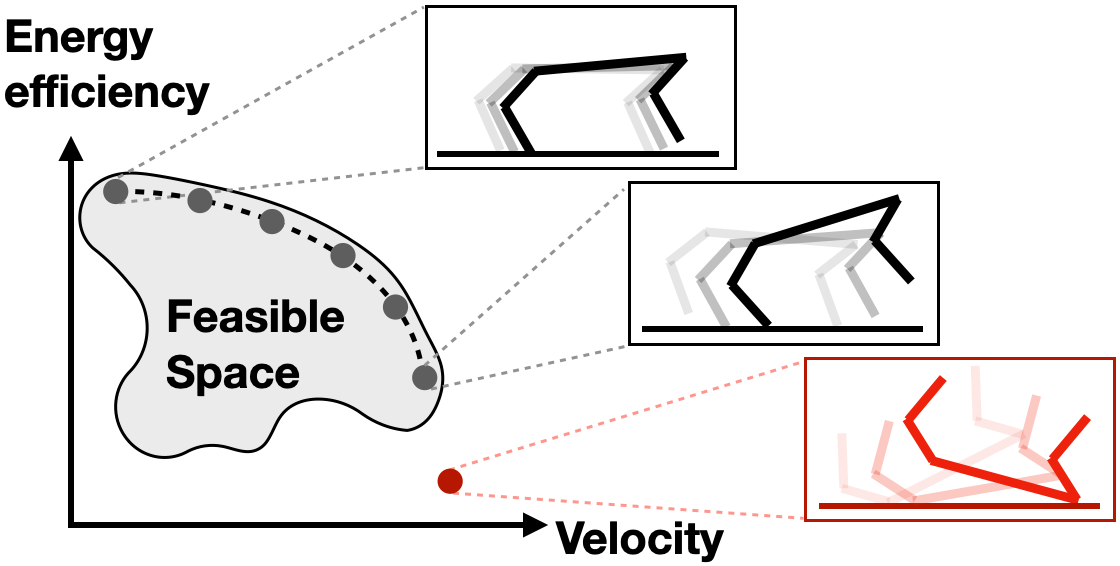}
        \caption{\small \textbf{Example of CMORL.} The robot aims to maximize energy efficiency and velocity while maintaining its balance to avoid falling.
        In order to consider such safety and multiple objectives, CMORL finds a set of feasible policies that are not dominated by other policies and satisfy constraints, which are indicated by the dashed line.}
        \label{fig: overview}
    \end{minipage}
    \hfill
    \begin{minipage}[c]{0.58\textwidth}
        \centering
        \includegraphics[width=0.8\textwidth]{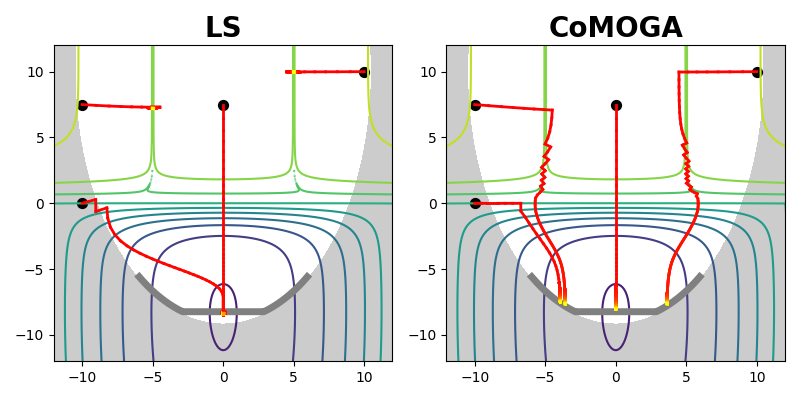}
        \vspace{-5pt}
        \caption{\small \textbf{Comparison between LS and the proposed method.} 
        Optimization trajectories are indicated in red, with initial points marked by black circles. 
        The contours illustrate the average values of the two objective functions, while the shaded area indicates regions where the constraints are violated. 
        The grey line represents the CP optimal set. 
        For some initial points, LS fails to reach the optimal set, whereas CoMOGA consistently finds it from any starting position.
        For more details, please see Appendix \ref{sec: toy example}.}
        \label{fig: toy example}
    \end{minipage}
    \vspace{-10pt}
\end{figure}

\section{Related Work}

\textbf{Multi-Objective RL.}
MORL aims to find a set of Pareto-optimal policies that are not dominated by any others.
Given that Pareto-optimal policies can be obtained with appropriate reward scalarization \citep{hayes2022practical}, many MORL algorithms utilize scalarization techniques.
Based on linear scalarization, \citet{yang2019envelop} introduced a new Bellman operator for MORL algorithms, which estimates a target Q value for a given preference by combining Q values from other preferences.
\citet{basaklar2023pdmorl} presented a similar Bellman operator and extended it to an actor-critic framework to enable its application in continuous action spaces.
For the continuous action spaces, \citet{lu2023multiobjective} proposed a SAC-based MORL algorithm, theoretically demonstrating that elements of the Pareto-optimal policy set can be achieved through linear scalarization.
Not only the linear scalarization, \citet{moffaert2013scalarization} proposed a nonlinear scalarization method, called \textit{Chebyshev} scalarization, which can cover a non-convex policy set.
To further improve performance, the evolution strategy (ES) can be applied as done in multi-objective optimization criteria \citep{deb2000nsga2}.
\citet{chen2020combining} and \citet{xu2020pgmorl} proposed methods that repetitively update a policy for each population using linear scalarization and produce a new generation of policy parameters.
There is also a different approach from scalarization. 
\citet{abdolmaleki2020mompo} have proposed a heuristic method of applying preference in the action distribution space instead of the reward space.
It is further extended to a CMORL algorithm in \citep{huang22lp3}.

\textbf{Constrained RL.}
Constrained RL is developed to explicitly consider safety or constraints in RL.
Based on the approach to handling constraints, constrained RL algorithms can be categorized into two types: primal-dual methods and primal methods.
The primal-dual methods, also called Lagrangian methods, are designed to address Lagrange dual problems. 
\citet{ding2022convergence} introduced a natural policy gradient-based primal-dual method and demonstrated its convergence to an optimal policy at a specified convergence rate. 
Another primal-dual method, proposed by \citet{bai2022achieving}, ensures that a trained policy results in zero constraint violations during evaluation. 
While these methods are straightforward to implement and can be integrated with standard RL algorithms, the training process can be unstable due to additional optimization variables (Lagrange multipliers) \citep{stooke2020pid}.
In contrast to the primal-dual method, the primal method addresses the constrained RL problem directly, so no additional optimization variables are required.
\citet{achiam2017cpo} introduced a trust region-based primal method. This approach linearly approximates a constraint within a trust region and updates the policy through a line search.
\citet{xu2021crpo} presented a natural policy gradient-based algorithm and demonstrated its convergence rate towards an optimal policy.
In multi-constrained settings, \citet{kim2023sdac} proposed a primal method to handle infeasible starting cases. 

\section{Background}

\textbf{Constrained Multi-Objective MDP.} 
We define a constrained multi-objective Markov decision process (CMOMDP) as a tuple represented by $\langle S$, $A$, $\mathcal{P}$, $R$, $C$, $\rho$, $\gamma\rangle$, where $S$ is a state space, $A$ is an action space, $\mathcal{P}: S \times A \times S \mapsto [0, 1]$ is a transition model, $R: S \times A \times S \mapsto \mathbb{R}^N$ is a vectorized reward function, $C: S \times A \times S \mapsto \mathbb{R}^M$ is a vectorized cost function, $\rho$ is an initial state distribution, and $\gamma$ is a discount factor.
The reward and cost function are bounded by $R_\mathrm{max}$.
In the CMOMDP setting, a policy $\pi(\cdot|s) \in \Pi$ can be defined as a state-conditional action distribution.
The objective function of the $i$th reward is defined as:
\begin{equation*}
J_{R_i}(\pi) := \mathbb{E}\left[\sum\nolimits_{t=0}^\infty \gamma^t R_i(s_t, a_t, s_{t+1})\right],
\end{equation*}
where $s_0 \sim \rho$, $a_t \sim \pi(\cdot|s_t)$, and $s_{t+1} \sim \mathcal{P}(\cdot|s_t, a_t)$ $\forall t$.
We also define the constraint function $J_{C_k}(\pi)$ by replacing the reward with the cost function.
A constrained multi-objective RL (CMORL) problem is defined as:
\begin{equation}
\label{eq: cmorl problem}
\mathrm{maximize}_\pi \; J_{R_i}(\pi) \; \forall i \in \{1, ..., N\} \quad \mathbf{s.t.} \; J_{C_k}(\pi) \leq d_k \; \forall k \in \{1, ..., M\},
\end{equation}
where $d_k$ is the threshold of the $k$th constraint.
Here, it is ambiguous to determine which policy is optimal when there are multiple objectives.
To address this, we instead define a set of optimal policies using the following notion, called \emph{constrained dominance} \citep{miettinen1999nonlinear}:
\begin{definition}[Constrained Dominance]
Given two policies $\pi_1, \pi_2 \in \{\pi\in \Pi | J_{C_k}(\pi) \leq d_k \; \forall k\}$, $\pi_1$ is dominated by $\pi_2$ if $J_{R_i}(\pi_1) \leq J_{R_i}(\pi_2) \; \forall i \in \{1, ..., N\}$. 
\end{definition}
A policy $\pi$ which is not dominated by any policy is called a \emph{constrained-Pareto (CP) optimal} policy, and the set of all CP optimal policies is called \emph{constrained-Pareto (CP) front} \citep{miettinen1999nonlinear}.
Additionally, a condition for gradient conflicts is defined as follows:
\begin{definition}[Gradient Conflict] 
Let $g_i$ be the gradient of the $i$th objective function, and let $g$ be the gradient used for policy updates. 
We say there is a gradient conflict if $\exists i$ such that $g_i^T g<0$.
\end{definition}
As studied in MTL \citep{liu2021cagrad, yu2020pcgrad} and shown in Figure \ref{fig: toy example}, gradient conflicts can cause convergence to local optima.
Thus, we aim to find the CP front while updating policies without gradient conflicts.

\textbf{Preference.}
In order to express a subset of the CP front, many existing MORL works \citep{xu2020pgmorl, alegre2023sample} introduce a preference space, defined as: $\Omega := \{\omega \in \mathbb{R}^N | \; \omega_i \geq 0, \; ||\omega||_\infty = 1 \}$,\footnote{It is common to define the preference space using the norm $||\omega||_1$, but for convenience, we employ $||\omega||_\infty$. Additionally, an one-to-one conversion exists between these definitions.} where a preference $\omega \in \Omega$ serves to specify the relative significance of each objective.
Subsequently, a specific subset of the Pareto front can be characterized by a function mapping from the preference space to the policy space.
We denote this function as a \textit{universal policy} $\pi(\cdot|s, \omega)$, which is characterized as an action distribution contingent upon both the state and preference.
Ultimately, our goal is to train the universal policy that covers the CP front as extensively as possible.


\section{Proposed Method}

In this section, our discussion is organized into four parts: \textbf{1)} building a constrained optimization problem to handle the multiple objectives with safety constraints, \textbf{2)} calculating a policy gradient by aggregating gradients of the objective and constraint functions for a single preference, \textbf{3)} introducing a method for training a universal policy covering a spectrum of preferences, and \textbf{4)} concluding with an analysis of the convergence properties of the proposed method.
Before that, we first parameterize the policy and define gradients of the objective and constraint functions, as well as a local region in the parameter space.
By representing the policy as $\pi_\theta$ with a parameter $\theta \in \Theta$, the gradients of the objective and constraint functions in (\ref{eq: cmorl problem}) are expressed as:
\begin{equation*}
g_i := \nabla_\theta J_{R_i}(\pi_\theta), \; b_k := \nabla_\theta J_{C_k}(\pi_\theta).
\end{equation*}
For brevity, we will denote the objectives as $J_{R_i}(\theta)$ and the constraints as $J_{C_k}(\theta)$.
The local region is then defined using a positive definite matrix $H$ as: $||\Delta \theta||_H \leq \epsilon$, where $||x||_H := \sqrt{x^T H x}$, and $\epsilon$ is the local region size, a hyperparameter.
$H$ can be the identity matrix or the Fisher information matrix, as in TRPO paper \citep{schulman2015trpo}.

\subsection{Transformation to Constrained Optimization}

We introduce a process of treating the CMORL problem as a constrained optimization problem (COP). 
To provide insight, we first present the process in a simplified problem where only the $i$th objective exists and then extend it to the original problem.
The problem of maximizing the $i$th objective within the local region is expressed as follows:
\begin{equation}
\label{eq: simplified problem}
\begin{aligned}
\mathrm{max}_{\Delta \theta} \; J_{R_i}(\theta_\mathrm{old} + \Delta \theta) \quad \mathbf{s.t.} \; ||\Delta \theta||_H \leq \epsilon.
\end{aligned}
\end{equation}
Assuming that $J_{R_i}(\theta)$ is linear with respect to $\theta$ within the local region, the above problem has the same solution as the following problem:
\begin{equation}
\label{eq: simplified converted problem}
\begin{aligned}
\mathrm{min}_{\Delta \theta} \; \Delta \theta^T H \Delta \theta \quad \mathbf{s.t.} \; e_i \leq J_{R_i}(\theta_\mathrm{old} + \Delta \theta) - J_{R_i}(\theta_\mathrm{old}),
\end{aligned}
\end{equation}
where $e_i := \epsilon ||g_i||_{H^{-1}}$.
It is shown that the solutions to (\ref{eq: simplified problem}) and (\ref{eq: simplified converted problem}) are identical in Appendix \ref{sec: proof of transformation}, and this finding confirms that objectives can be converted into constraints in this way.
Consequently, the CMORL problem can be transformed as follows:
\begin{equation}
\label{eq: transformed CMORL problem}
\mathrm{min}_{\Delta \theta} \Delta \theta^T H \Delta \theta \quad \mathbf{s.t.} \; \omega_i e_i \leq J_{R_i}(\theta_\mathrm{old} + \Delta \theta) - J_{R_i}(\theta_\mathrm{old}) \; \forall i, \; J_{C_k}(\theta_\mathrm{old} + \Delta \theta) \leq d_k \; \forall k,
\end{equation}
where $e_i$ in (\ref{eq: simplified converted problem}) is scaled to $\omega_i e_i$ to reflect the given preference $\omega$.
We will show that the policy updated using (\ref{eq: transformed CMORL problem}) converges to a CP optimal policy of (\ref{eq: cmorl problem}) in Section \ref{sec: convergence analysis} under mild assumptions.

\subsection{Gradient Aggregation}

Now, we calculate a policy gradient by solving (\ref{eq: transformed CMORL problem}) through a linear approximation and update a policy within the local region to reduce approximation errors.
First, the problem (\ref{eq: transformed CMORL problem}) can be linearly approximated as the following quadratic programming (QP) problem:
\begin{equation}
\label{eq: gradient aggregation}
\bar{g}_\omega^\mathrm{ag} := \mathrm{argmin}_{\Delta\theta} \Delta\theta^T H \Delta\theta \quad \mathbf{s.t.} \; \omega_i e_i \leq g_i^T\Delta\theta \; \forall i, \; b_k^T\Delta\theta + J_{C_k}(\theta_\mathrm{old}) \leq d_k \; \forall k,
\end{equation}
where the gradients of the objective and constraint functions, $g_i$ and $b_k$, are aggregated into $\bar{g}_\omega^\mathrm{ag}$.
The gradient is then clipped to ensure that the policy is updated within the local region, as follows:
\begin{equation*}
g_\omega^\mathrm{ag} := \min(1, \epsilon /||\bar{g}_\omega^\mathrm{ag}||_H) \bar{g}_\omega^\mathrm{ag},
\end{equation*}
where the policy will be updated by $\theta_\mathrm{old} + g_\omega^\mathrm{ag}$.
However, due to the linear approximation, the updated policy can violate the safety constraints.
To address this issue, we take a recovery step by solving the following QP problem, which only consists of the safety constraints, when the current policy violates any safety constraints:
\begin{equation}
\label{eq: violated gradient aggregation}
\bar{g}_\omega^\mathrm{ag} := \mathrm{argmin}_{\Delta\theta} \Delta\theta^T H \Delta\theta \quad \mathbf{s.t.} \; b_k^T\Delta\theta + J_{C_k}(\theta_\mathrm{old}) \leq d_k \; \forall k. 
\end{equation}
As a result, the process for obtaining the policy gradient can be summarized as follows, which is named \emph{constrained multi-objective gradient aggregator (CoMOGA)}:
\begin{equation*}
\begin{aligned}
\bar{g}_\omega^\mathrm{ag} &= \begin{cases}
\mathrm{argmin}_{\Delta\theta} \Delta\theta^T H \Delta\theta \; \mathbf{s.t.} \; \omega_i e_i \leq g_i^T\Delta\theta, \; b_k^T\Delta\theta + J_{C_k}(\theta_\mathrm{old}) \leq d_k & \text{if} \; J_{C_k}(\theta_\mathrm{old}) \leq d_k \; \forall k, \\
\mathrm{argmin}_{\Delta\theta} \Delta\theta^T H \Delta\theta \; \mathbf{s.t.} \; b_k^T\Delta\theta + J_{C_k}(\theta_\mathrm{old}) \leq d_k & \text{otherwise,} \\
\end{cases} \\
g_\omega^\mathrm{ag} &= \min(1, \epsilon /||\bar{g}_\omega^\mathrm{ag}||_H) \bar{g}_\omega^\mathrm{ag},
\end{aligned}
\end{equation*}
and it is illustrated in Figure \ref{fig: CoMOGA}.
Note that the aggregated gradient satisfies  $0 \leq g_i^T g_\omega^\mathrm{ag}$ $\forall i$ due to the transformed constraints in (\ref{eq: gradient aggregation}), which ensures that the aggregated gradient does not conflict with any objective functions.

\begin{figure}[t]
\centering
\includegraphics[width=0.9\linewidth]{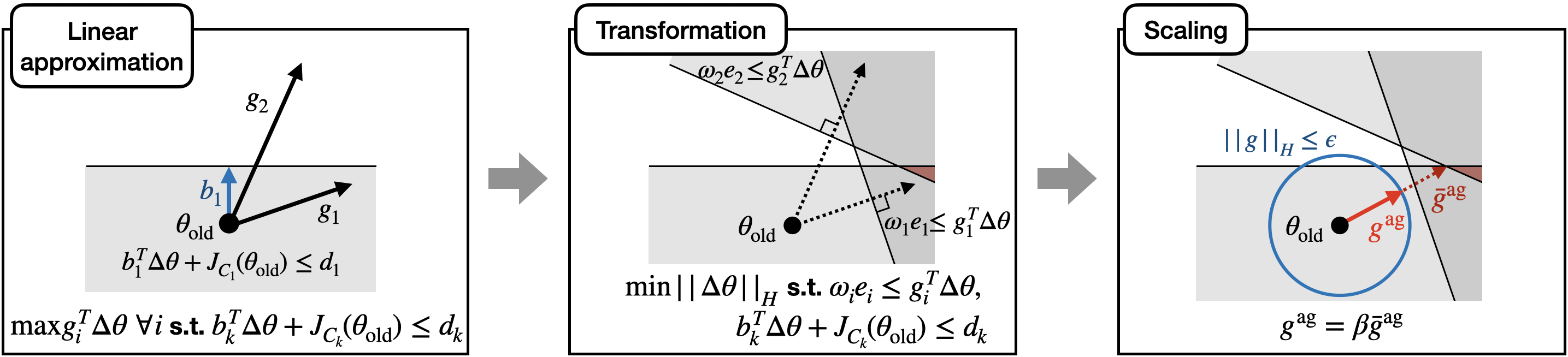}
\caption{\small
\textbf{Process of CoMOGA}. 
We visualize the process in the parameter space, and the gray areas represent constraints.
\textbf{(Linear approximation)}
CoMOGA linearly approximates the original CMORL problem in (\ref{eq: cmorl problem}). 
The gradients of the objective and constraint functions are visualized as black and blue arrows, respectively. 
\textbf{(Transformation)}
The objectives are converted to constraints as described in (\ref{eq: gradient aggregation}). 
The intersection of all constraints is shown as the red area.
\textbf{(Scaling)}
The solution of the transformed problem, $\bar{g}^\mathrm{ag}$, is scaled to ensure that the updated policy remains within the local region.}
    \label{fig: CoMOGA}
\vspace{-10pt}
\end{figure}

\subsection{Update Rule for Universal Policy}
We have introduced a gradient aggregation method named \textit{CoMOGA} for updating a policy for a given single preference.
This section presents a method for updating a universal policy that can cover the entire preference space.
Given that a policy is updated to $\theta_\mathrm{old}+g_\omega^\mathrm{ag}$ when a preference $\omega$ is specified, the ideal universal policy should satisfy the following conditions:
\begin{equation}
\bar{\pi}(a|s, \omega) = \pi_{\theta_\mathrm{old} + g_\omega^\mathrm{ag}}(a|s, \omega) \quad \forall a \in A,  \; \forall s \in S, \;\text{and}\; \forall \omega \in \Omega,
\end{equation}
where we denote $\pi_{\theta_\mathrm{old} + g_\omega^\mathrm{ag}}$ as an intermediate policy $\bar{\pi}_\omega$.
This can be interpreted as the KL divergence between the universal policy and the intermediate policy being zero.
Based on this property, we introduce a practical method to achieve the universal policy by minimizing the following loss:
\begin{equation}
\label{eq: universal policy update}
\begin{aligned}
\mathrm{min}_\theta \; \mathbb{E}_{(s, \omega) \sim \mathcal{D}} \left[ D_\mathrm{KL}(\bar{\pi}_\omega(\cdot|s, \omega)||\pi_\theta(\cdot|s, \omega)) \right],
\end{aligned}
\end{equation}
where $\mathcal{D}$ is a replay buffer.
By minimizing the above loss, we can combine policies for each preference into a single universal policy.
After the training is completed, users can utilize various policies by providing preferences to the trained universal policy without additional training.

\subsection{Convergence Analysis}
\label{sec: convergence analysis}

In this section, we analyze the convergence properties of the proposed method. 
Specifically, it can be shown that CoMOGA achieves CP optimal convergence with minor modifications if $H$ is the Fisher information matrix. 
To this end, we first introduce the necessary conditions for the policy gradient method to ensure CP optimal convergence.
We then show that the modified CoMOGA satisfies these necessary conditions.
Before that, we first introduce the following assumption:
\begin{assumption}[Slater's condition]
There exists a feasible policy $\pi_f$ such that $J_{C_k}(\pi_f) \leq d_k - \eta$ $\forall k$, where $\eta > 0$ is a positive number.
\end{assumption}
The Slater's condition is a common assumption in safe RL to ensure optimal convergence \citep{kim2023sdac, bai2022achieving, ding2022convergence}.
We also assume that the state and action spaces are finite.
Now, we introduce a generalized version of the policy gradient method that guarantees optimal convergence.
\begin{restatable}{theorem}{GeneralPolicyUpdate}
\label{thm: generalized policy update rule}
Assume that sequences $\nu^a_{t,i}$, $\nu^b_{t,i}$, $\lambda^a_{t,k}$, $\lambda^b_{t,k} \in [0, \lambda_\mathrm{max}]$ are given for all $i$ and $k$, where $\sum_k\lambda^b_{t,k} = 1$, $\lambda^b_{t,k}(J_{C_k}(\theta_t) - d_k) \geq 0$, $\sum_i\nu^a_{t,i}=1$, and $\nu^a_{t,i}$ converges to a specific point $\bar{\nu}^a_i$. 
Then, if a policy is updated according to the following rule, it converges to a CP optimal policy of (\ref{eq: cmorl problem}):
\begin{equation}
\label{eq: generalized policy update rule}
\theta_{t+1} = \theta_t + \begin{cases}
\alpha_t F^\dagger(\theta_t) \left( \sum_{i=1}^N \nu^a_{t,i}g_i - \alpha_t\sum_{k=1}^M \lambda^a_{t,k} b_k \right) & \text{if} \; J_{C_k}(\theta_t) \leq d_k \; \forall k, \\
\alpha_t F^\dagger(\theta_t) \left(\alpha_t \sum_{i=1}^N \nu^b_{t,i} g_i - \sum_{k=1}^M \lambda^b_{t,k} b_k \right) & \text{otherwise}, \\
\end{cases}
\end{equation}
where $g_i = \nabla J_{R_i}(\theta_t)$, $b_k = \nabla J_{C_k}(\theta_t)$, $\alpha_t$ is a step size satisfying Robbins-Monro condition \citep{robbins1951stochastic}, and $F^\dagger$ represents the pseudo-inverse of the Fisher information matrix.
\end{restatable}
The proof is provided in Appendix \ref{sec: proof of generalized policy rule}.
The sequences $\nu^a_{t,i}$, $\nu^b_{t,i}$, $\lambda^a_{t,k}$, and $\lambda^b_{t,k}$ in Theorem \ref{thm: generalized policy update rule} represent the weights assigned to the gradients of the objective and constraint functions at each update step, and a new CMORL algorithm can be developed based on how the sequences are set. 
Similarly, we will show that CoMOGA can be formulated as (\ref{eq: generalized policy update rule}) by identifying these sequences with minor modifications.
Since the quadratic programming holds strong duality, we can express the solution of (\ref{eq: gradient aggregation}) as
$\bar{g}_\omega^\mathrm{ag} = H^{-1}(\sum_{i=1}^N \nu^*_i g_i - \sum_{k=1}^M \lambda^*_k H^{-1}b_k)$, where $\nu^*_i$ and $\lambda_k^*$ are the optimal Lagrange multipliers of the following dual problem:
\begin{equation*}
\begin{aligned}
\nu^*, \lambda^* = \underset{\nu \geq 0, \lambda \geq 0}{\mathrm{argmax}} &-\frac{1}{2} \left( \sum\nolimits_i \nu_i g_i - \sum\nolimits_k \lambda_k b_k \right)^T H^{-1} \left( \sum\nolimits_i \nu_i g_i - \sum\nolimits_k \lambda_k b_k \right) + \sum\nolimits_i \nu_i \omega_i e_i \\
&+ \sum\nolimits_k \lambda_k (J_{C_k}(\theta_\mathrm{old}) - d_k).
\end{aligned}
\end{equation*}
In case of the dual problem has no solution, we set $\nu_i^*$ and $\lambda_k^*$ to be $\bar{g}_\omega^\mathrm{ag} = 0$.
For the solution of (\ref{eq: violated gradient aggregation}), the same process is applied to achieve $\lambda_k^*$.
Then, the CoMOGA is modified as follows:
\begin{equation}
\label{eq: modified CoMOGA}
\begin{aligned}
& \bar{g}_\omega^\mathrm{ag} = H^{-1}\cdot \begin{cases}
\sum_{i=1}^N \frac{\nu_i^*}{\sum_j \nu_j^*}g_i - \epsilon_t\sum_{k=1}^M \min(\frac{\lambda_k^*}{\epsilon_t\sum_j \nu_j^*}, \lambda_\mathrm{max}) b_k &\!\! \text{if} \; J_{C_k}(\theta_t) \leq d_k \; \forall k, \\
- \sum_{k=1}^M (\lambda_k^*/(\sum_j \lambda_j^*)) b_k &\!\! \text{otherwise}, \\
\end{cases} \\
& g_\omega^\mathrm{ag} = \epsilon_t \bar{g}_\omega^\mathrm{ag} / \min(\max(||\bar{g}_\omega^\mathrm{ag}||_H, g_\mathrm{min}), g_\mathrm{max}),
\end{aligned}
\end{equation}
where $g_\mathrm{min}$, $g_\mathrm{max}$, $\lambda_\mathrm{max} \in \mathbb{R}_{\geq 0}$ are hyperparameters, and $\epsilon_t$ is the local region size satisfying Robbins-Monro condition.
Finally, we can show optimal convergence of the modified CoMOGA.
\begin{restatable}{theorem}{CoMOGA}
\label{thm: convergence of modified CoMOGA}
If $H$ is the Fisher information matrix and a policy $\pi_{\theta_t}$ is updated as $\theta_{t+1} = \theta_t + g_\omega^\mathrm{ag}$, where $g_\omega^\mathrm{ag}$ is defined in (\ref{eq: modified CoMOGA}), it converges to a CP optimal policy of (\ref{eq: cmorl problem}).
\end{restatable}
Theorem \ref{thm: convergence of modified CoMOGA} can be proved by identifying sequences $\nu_{t,i}^a$, $\nu_{t,k}^b$, $\lambda_{t,i}^a$, $\lambda_{t,k}^b$ that satisfy the conditions mentioned in Theorem \ref{thm: generalized policy update rule}, and details are provided in Appendix \ref{sec: proof of CoMOGA}.


\section{Practical Implementation}
\label{sec: prac implementation}


We need to calculate the gradient of the objective and constraint functions to perform CoMOGA and the universal policy update.
To do that, we use reward and cost critics $V_{R, \psi_R}^\pi$, $V_{C, \psi_C}^{\pi}$, where $\psi_R$ and $\psi_C$ are the critic network parameters, to estimate the objectives and constraint function.
The critics are trained by minimizing the following loss functions:
\begin{equation}
\label{eq: critic loss}
\begin{aligned}
\mathcal{L}(\psi_R) &:= \underset{(s,a,s', \omega) \sim \mathcal{D}}{\mathbb{E}}\left[ \sum\nolimits_{i=1}^N(Y_{R_i} - V^\pi_{R_i, \psi_R}(s, a, \omega))^2 \right], \\
\end{aligned}\end{equation}
where the target $Y_{R_i}$ can be calculated either as $R_i(s,a,s') + \gamma V^\pi_{R_i, \psi_R}(s', a', \omega)$ with $a' \sim \pi(\cdot|s', \omega)$ or using Retrace($\lambda$) \citep{munos2016retrace}, and the same process is applied to the cost critics.
Given a preference $\omega$, the objective and constraint functions can be estimated as follows:
\begin{equation}
\label{eq: obj and cons estimation}
\begin{aligned}
J_{R_i}(\theta) &\approx \underset{\begin{subarray}{c}s \sim \mathcal{D},\\ a \sim \pi_\theta(\cdot|s, \omega)\end{subarray}}{\mathbb{E}}\left[V_{R_i, \psi_R}^\pi(s, a, \omega) \right],\; J_{C_k}(\theta) \approx \underset{\begin{subarray}{c}s \sim \mathcal{D},\\ a \sim \pi_\theta(\cdot|s, \omega)\end{subarray}}{\mathbb{E}} \left[V_{C_k, \psi_C}^\pi(s, a, \omega) \right].
\end{aligned}
\end{equation}
Then, we can compute the policy gradients using the critic networks through the reparameterization trick, as done in the SAC paper \cite{haarnoja2018sac}.
Finally, the proposed method is summarized in Algorithm \ref{algo:proposed method}.


\begin{algorithm}[t]
\small
\caption{Policy Update Using CoMOGA}
\label{algo:proposed method}
\begin{algorithmic}
\STATE {\bfseries Input:} Policy parameter $\theta$, reward critic parameter $\psi_R$, cost critic parameter $\psi_C$, replay buffer $\mathcal{D}$.
\FOR{epochs$\;=1$ {\bfseries to} $E$}
    \FOR {rollouts$\;=1$ {\bfseries to} $L$}
        \STATE Sample a preference $\omega \sim \Omega$.
        \STATE Collect a trajectory $\tau = \{(s_t, a_t, r_t, c_t, s_{t+1})\}_{t=1}^T$ by using $\pi_\theta(\cdot|\cdot, \omega)$ and store $(\tau, \omega)$ in $\mathcal{D}$.
    \ENDFOR
    \STATE Update reward and cost critic networks using Equation (\ref{eq: critic loss}).
    \FOR{$p=1$ {\bfseries to} $P$}
        \STATE Sample $(\tau, \omega) \sim D$, and estimate the objectives and constraints using Equation (\ref{eq: obj and cons estimation}).
        \STATE Calculate the aggregated gradient $g^\mathrm{ag}_{\omega}$ using Equation (\ref{eq: modified CoMOGA}), and set $\theta_\omega = \theta_\mathrm{old} + g^\mathrm{ag}_{\omega}$.
        \STATE Store an intermediate policy $\bar{\pi}_\omega = \pi_{\theta_\omega}$.
    \ENDFOR
    \STATE Update the universal policy using Equation (\ref{eq: universal policy update}).
\ENDFOR
\end{algorithmic}
\end{algorithm}

\section{Experiments}
\label{sec: experiment}
This section evaluates the proposed method and baselines across various tasks with and without constraints. 
First, we explain how methods are evaluated on the tasks and then present the CMORL baselines.
Subsequently, each task is described, and the results are analyzed.
Finally, we conclude this section with ablation studies of the proposed method.

\subsection{Evaluation Metrics and Baselines}

\textbf{Metrics.}
Once the universal policy is trained, an estimated CP front, denoted as $P\subset \mathbb{R}^N$, can be obtained by calculating the objective values $(J_{R_1}(\pi), ..., J_{R_N}(\pi))$ for a range of preferences. 
Given that CMORL aims to approximate the ground truth CP front, it is essential to assess how close the estimated CP front is to the ground truth, and it can be realized by measuring the coverage and density of the estimated front.
To this end, we use the \textit{hypervolume} (HV) for coverage measurement and \emph{sparsity} (SP) for density estimation, which are commonly used in many MORL algorithms \citep{xu2020pgmorl, basaklar2023pdmorl}.
Given an estimated CP front $P \subset \mathbb{R}^N$ and a reference point $r \in \mathbb{R}^N$, the hypervolume (HV) is defined as follows \citep{zitzler1999hv}:
\begin{equation}
\label{eq: hypervolume}
\begin{aligned}
\mathrm{HV}(P, r) := \int_{\mathbb{R}^N} \mathbf{1}_{Q(P, r)} (z) dz,
\end{aligned}
\end{equation}
where $Q(P,r) := \{z|\exists p \in P \; \text{s.t.} \; r \preceq z \preceq p\}$, and the metric represents the volume of the area surrounding the reference point and the estimated CP front. 
For a detailed explanation of this metric, including a visual description and how the reference point is set, please refer to Appendix \ref{sec: performance metric}.
Sparsity (SP), on the other hand, measures how uniformly the estimated CP fronts are distributed and is defined as the average of the squared distances between elements of the CP fronts \citep{xu2020pgmorl}.
However, the original definition exhibits a correlation with HV, where a larger HV corresponds to a larger SP. 
To eliminate this inherent correlation, we propose to use the following normalized version of SP:
\begin{equation}
\label{eq: normalized sparsity}
\overline{\mathrm{SP}}(P) := \frac{1}{|P| - 1}\sum_{j=1}^{N}\sum_{i=1}^{|P| - 1} \left(\frac{\tilde{P}_j[i] - \tilde{P}_j[i+1]}{\mathrm{max}_k \tilde{P}_j[k] - \mathrm{min}_k \tilde{P}_j[k]}\right)^2,
\end{equation}
where $\tilde{P}_j := \mathrm{Sort}(\{p[j]|\forall p \in P\})$, and $\tilde{P}_j[i]$ is the $i$th element in $\tilde{P}_j$.
By normalizing the distance using the minimum and maximum values of the CP fronts, we can remove the correlation with HV while preserving the ability to measure sparsity.

\textbf{Baselines.}
We aim to compare the proposed method with various existing CMORL algorithms; however, to the best of our knowledge, \textit{LP3} \citep{huang22lp3} is the only existing CMORL algorithm.
To provide more comprehensive baselines, we include MORL algorithms that are based on policy gradient approaches, which can be adapted to handle constraints using the Lagrangian method. 
There are two state-of-the-art policy gradient-based MORL algorithms: \textit{PD-MORL} \citep{basaklar2023pdmorl}, which is based on TD3 \citep{fujimoto2018td3}, and \textit{CAPQL} \citep{lu2023multiobjective}, which is based on SAC \citep{haarnoja2018sac}.
We have extended them to CMORL by handling constraints using the Lagrangian method.
For implementation details, please refer to Appendix \ref{sec: ls}.

\subsection{Legged Robot Locomotion}

The legged robot locomotion tasks \citep{kim2023sdac} are to control a quadrupedal or bipedal robot to follow randomly given commands while satisfying three constraints: keeping 1) body balance, 2) CoM height, and 3) pre-defined foot contact timing.
Since the original tasks were designed to have a single objective, we have modified the tasks to have two objectives: 1) matching the current velocity with the command and 2) minimizing energy consumption.
Please see Appendix \ref{sec: experiment detail} for more details.

The evaluation results are presented in Figure \ref{fig: legged robot results}, and the estimated CP fronts are visualized in Appendix \ref{sec: additional results}.
Across all tasks, CoMOGA achieves the highest HV and the lowest SP while satisfying all constraints. 
Notably, CoMOGA surpasses the baselines in the bipedal task. 
This suggests that preventing gradient conflicts is particularly effective for dynamically unstable robots, such as bipedal robots, because these robots require a careful balance of multiple objectives to attain desired behaviors.
LP3 shows performance similar to the proposed method in the quadrupedal task but underperforms in the bipedal task. 
This may be attributed to its less effective handling of gradient conflicts.

\begin{figure}[t]
    \centering
    \includegraphics[width=1.0\textwidth]{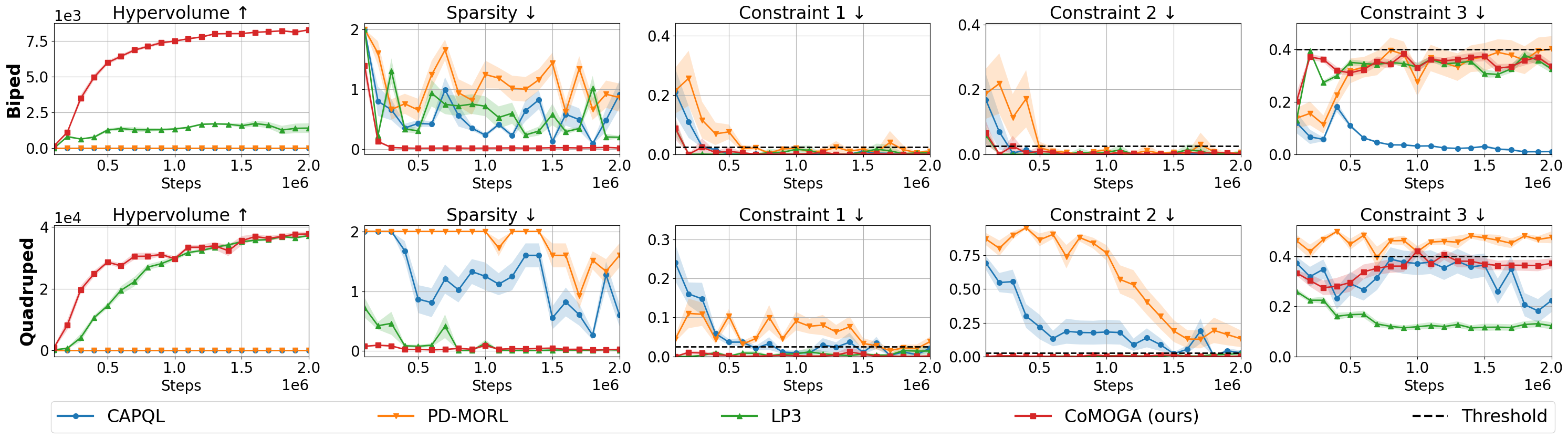}
    \vspace{-15pt}
    \caption{\small
    Evaluation results of the legged robot locomotion tasks.
    The upper row shows results for the bipedal robot, while the lower row is for the quadrupedal robot.
    All algorithms are evaluated at every $10^5$ steps. 
    The bold lines and shaded areas represent the mean and quarter-scaled standard deviation of results from five random seeds, respectively.
    The black dotted lines in constraint graphs indicate the thresholds.
    }
    \label{fig: legged robot results}
    \vspace{-10pt}
\end{figure}

\subsection{Safety Gymnasium}

We utilize single-agent and multi-agent goal tasks in the Safety Gymnasium environments \citep{ji2023safetygym}.
These tasks require navigating robots to designated goals while avoiding obstacles. 
The single-agent goal tasks have two objectives: 1) reaching goals as many times as possible within a time limit and 2) maximizing energy efficiency, along with a single constraint to avoid collisions with obstacles.
The multi-agent goal tasks have two objectives and two constraints, which are to maximize goal achievement and avoid collisions for two robots, respectively.
In both tasks, point and car robots are used.
For details on the tasks and hyperparameter settings, please refer to Appendix \ref{sec: experiment detail}.

The results are presented in Figure \ref{fig: safety gym results}.
CoMOGA shows the highest HV in all tasks and the lowest SP in two out of four tasks, while satisfying all constraints.
Additionally, the low volatility in the evaluation graph indicates that the proposed method can stably handle both constraints and objectives.
LP3 and PD-MORL violate constraints in at least one task, and CAPQL satisfies all constraints but shows weak performance in the single agent tasks, implying the difficulty of handling constraints and objectives simultaneously.

\begin{figure}[t]
    \centering
    \begin{subfigure}[b]{1.0\textwidth}
        \centering
        \includegraphics[width=1.0\textwidth]{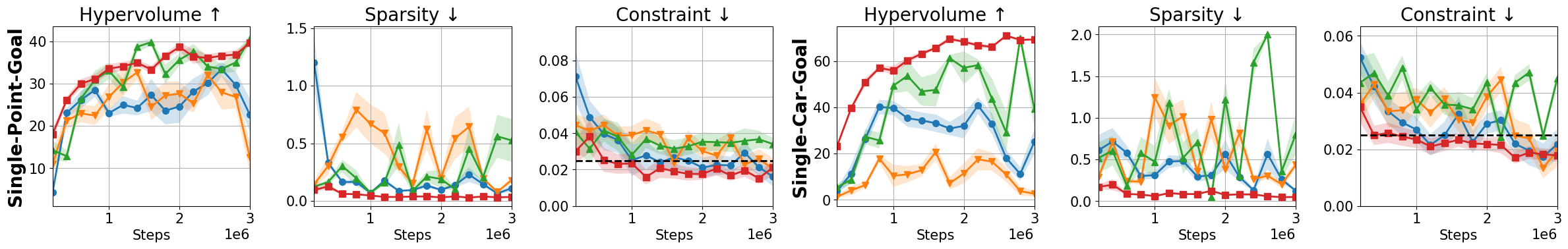}
        \vspace{-15pt}
        \caption{
            \small Single-agent tasks.
            The left shows the results for the point robot, and the right is for the car robot.
        }
    \end{subfigure}
    \hfill
    \begin{subfigure}[b]{1.0\textwidth}
        \centering
        \includegraphics[width=1.0\textwidth]{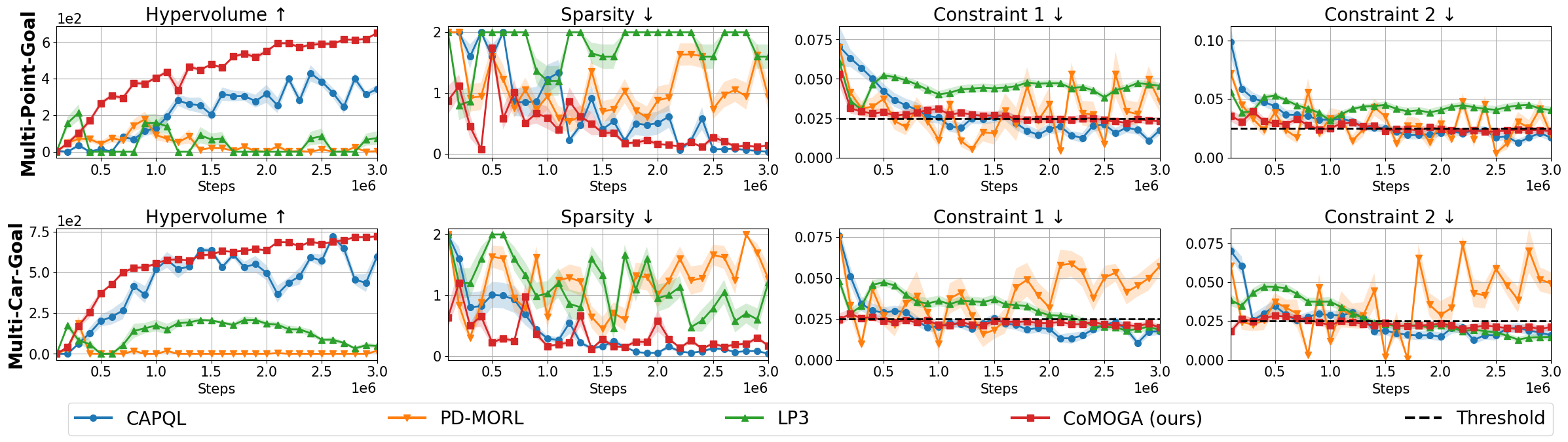}
        \vspace{-15pt}
        \caption{
            \small Multi-agent tasks.        
            The upper row shows the results for the point robot, and the lower is for the car robot.
        }
    \end{subfigure}
    \vspace{-15pt}
    \caption{\small Evaluation results of the Safety Gymnasium tasks.
    }
    \label{fig: safety gym results}
    \vspace{-10pt}
\end{figure}

\subsection{Multi-Objective Gymnasium}

We conduct experiments in the Multi-Objective (MO) Gymnasium \citep{alegre2022bnaic}, which is a well-known MORL environment, to examine whether the proposed method also performs well on unconstrained MORL tasks.
We use the MuJoCo tasks with continuous action spaces in the MO Gymnasium, and there are six tasks available: Hopper, Humanoid, Half-Cheetah, Walker2d, Ant, and Swimmer.
Each task has three, two, two, two, three and two objectives, respectively.
Details are provided in Appendix \ref{sec: experiment detail}.

The results are presented in Figure \ref{fig: mo gym results}, and the proposed method achieves the highest HV in three tasks and the lowest SP in five out of six tasks. 
Especially, CoMOGA exhibits outstanding performance in the hopper and humanoid tasks, which suggests that the conflict-averse strategy is effective for dynamically unstable robots even in unconstrained MORL settings.
However, all algorithms except LP3 show high HV results in the walker2d task, which is also dynamically unstable.
It can be considered that they have reached the global optimum due to the well-designed reward functions.

\begin{figure}[t]
    \centering
    \includegraphics[width=1.0\textwidth]{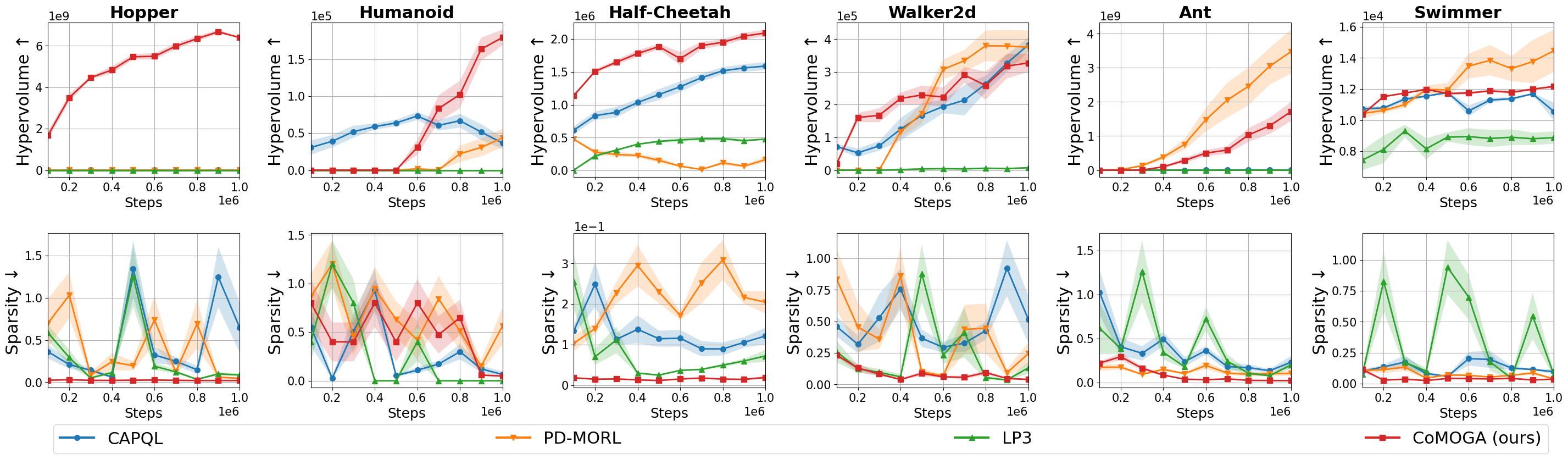}
    \vspace{-15pt}
    \caption{\small Evaluation results of the MO Gymnasium tasks. Each column shows the results of the task corresponding to its title.}
    \label{fig: mo gym results}
    \vspace{-10pt}
\end{figure}

\subsection{Ablation Study}

Some might suggest that conflict aversion in CMORL can be achieved by simply integrating MTL algorithms with the Lagrangian method. 
To answer this question, we compare a method that combines a conflict-averse MTL algorithm, CAGrad \citep{liu2021cagrad}, and the Lagrangian approach against CoMOGA in the legged locomotion tasks. 
The results are shown in Figure \ref{fig: ablation results}. 
CAGrad+Lagrangian violates the second constraint in the quadrupedal robot, and while it satisfies all constraints in the bipedal robot, HV and SP metrics are significantly lower than CoMOGA.
These results demonstrate that the proposed method can handle both constraints and gradient conflicts effectively.

\begin{figure}[t]
    \centering
    \includegraphics[width=1.0\textwidth]{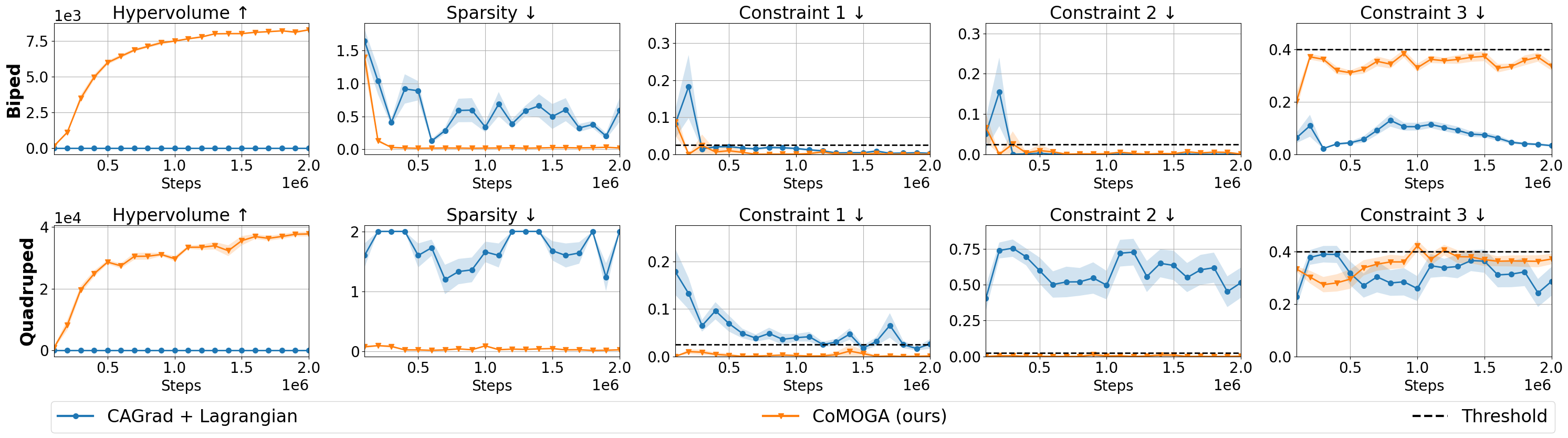}
    \vspace{-15pt}
    \caption{\small Results of the ablation study in the legged robot locomotion tasks.}
    \label{fig: ablation results}
    \vspace{-10pt}
\end{figure}

\section{Conclusions and Limitations}
\label{sec: conclusion}


We have introduced a CMORL algorithm called CoMOGA to maximize multiple objectives while satisfying safety constraints.
The proposed method is based on a novel transformation process that converts objectives into constraints, which enables the avoidance of gradient conflicts among multiple objectives and the handling of constraints without additional optimization variables.
We show that the proposed method converges to a CP optimal policy in tabular settings and demonstrate that it achieves outstanding performance in HV and SP metrics with constraint satisfaction through various experiments.
Since the proposed method updates the policy in a manner that prevents any objective functions from decreasing, the set of trained policies may cover a narrow portion of the CP front. 
Future research can address this limitation by mitigating the condition of gradient conflicts to better balance multiple objectives and expand coverage.
Specifically, by satisfying the conditions of the proposed generalized policy update rule in Theorem \ref{thm: generalized policy update rule}, various CMORL algorithms can be developed with convergence guarantees, which can accelerate future work.



\bibliographystyle{plainnat}
\bibliography{main}


\newpage
\appendix

\section{Toy Example}
\label{sec: toy example}


\begin{figure}[htb]
    \centering
    \begin{subfigure}[b]{0.4\textwidth}
        \centering
        \includegraphics[width=1.0\textwidth]{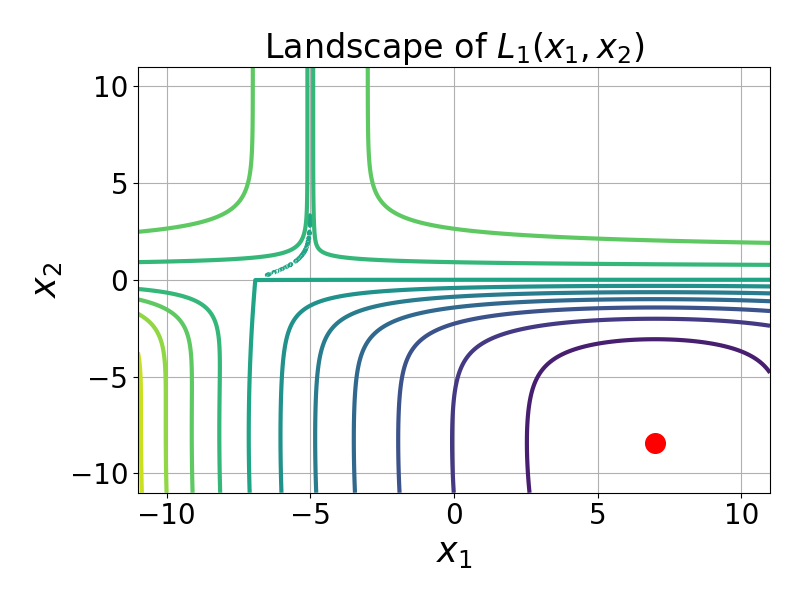}
    \end{subfigure}
    \hspace{15pt}
    \begin{subfigure}[b]{0.4\textwidth}
        \centering
        \includegraphics[width=1.0\textwidth]{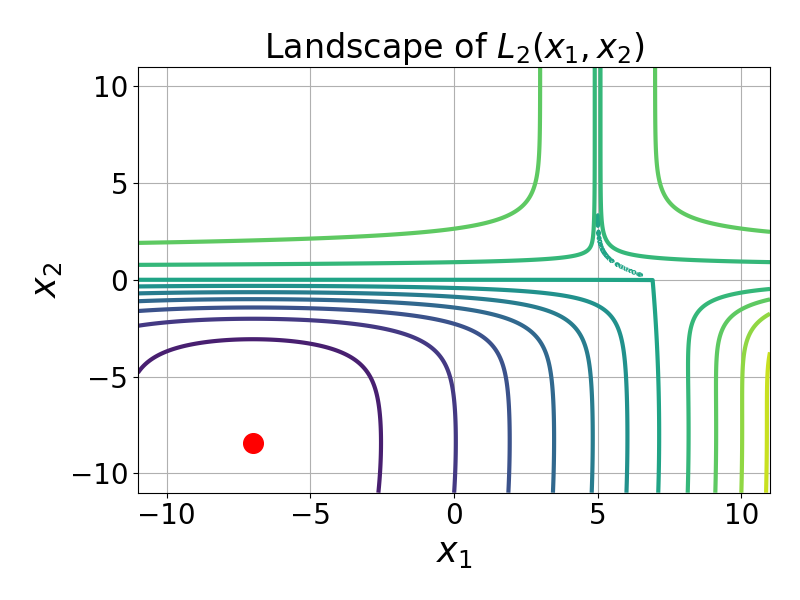}
    \end{subfigure}
    \caption{Landscape of the objective functions of the toy example.}
    \label{fig: landscape of toy example}
\end{figure}

In this section, we describe the details of the toy example, which is initially defined in \citep{liu2021cagrad}.
Objective functions of the toy example are formulated as follows:
\begin{equation*}
\begin{aligned}
L_1(x_1, x_2) &:= c_1(x_2)f_1(x_1, x_2) + c_2(x_2)g_1(x_1, x_2), \\
L_2(x_1, x_2) &:= c_1(x_2)f_2(x_1, x_2) + c_2(x_2)g_2(x_1, x_2), \; \text{where} \\
c_1(x) &:= \max(\mathrm{tanh}(0.5x), 0), \\
c_2(x) &:= \max(\mathrm{tanh}(-0.5x), 0), \\
f_1(x_1, x_2) &:= \log( \max(|0.5(-x_1 - 7) - \mathrm{tanh}(-x_2)|, 0.000005)) + 6, \\
f_2(x_1, x_2) &:= \log( \max(|0.5(-x_1 + 3) + \mathrm{tanh}(-x_2 + 2)|, 0.000005)) + 6, \\
g_1(x_1, x_2) &:= ((-x_1 + 7)^2 + 0.1(-x_2 - 8)^2)/10- 20, \\
g_2(x_1, x_2) &:= ((-x_1 - 7)^2 + 0.1(-x_2 - 8)^2)/10- 20, \\
\end{aligned}
\end{equation*}
where $L_1$ and $L_2$ are the objective functions.
The landscape of the objective functions are presented in Figure \ref{fig: landscape of toy example}.
In addition, we define a constraint function, which is defined as follows:
\begin{equation*}
C(x_1, x_2) := x_1^2 + 0.3(x_2 - 10)^2 - 10.5^2.
\end{equation*}
Finally, the toy example is formulated as follows:
\begin{equation*}
\min_{x_1, x_2} L_1(x_1, x_2) \; \text{and} \; L_2(x_1, x_2) \quad \textbf{s.t.} \; C(x_1, x_2) \leq 0.
\end{equation*}

We set the preference as $(0.5, 0.5)$ for both algorithms, LS and CoMOGA, and use the following four initial points: $(-10, 0)$, $(-10, 7.5)$, $(0, 7.5)$, and $(10, 10)$.
The optimization trajectories for each method are shown in Figure \ref{fig: toy example}.
For the initial point $(-10, 7.5)$, the gradient of $L_1$ is much larger than $L_2$.
Thus, addressing the toy example with LS can be significantly influenced by $L_1$ and eventually converges to a local optimal point of $L_1$.
To resolve this issue, avoiding gradient conflicts can prevent policy gradients from being dominated by a specific objective, which helps them converge to an optimal point.

\newpage
\section{Proofs}
\label{sec: proof}

\subsection{Proof of Transformation Process}
\label{sec: proof of transformation}

In this section, we prove that the solutions of (\ref{eq: simplified problem}) and (\ref{eq: simplified converted problem}) are equivalent.
Under the linear assumption, we can express the $i$th objective function as $g_i^T \Delta \theta + J_{R_i}(\theta_\mathrm{old})$.
Then, (\ref{eq: simplified problem}) can be rewritten as follows:
\begin{equation*}
\max_{\Delta \theta} \; g_i^T\Delta \theta + J_{R_i}(\theta_\mathrm{old}) \;\; \mathrm{s.t.} \; \Delta \theta^TH \Delta \theta \leq \epsilon^2.
\end{equation*}
As the strong duality holds, we will find the solution by solving the following Lagrange dual problem:
\begin{equation*}
\max_{\lambda \geq 0} \min_{\Delta \theta} \; -g_i^T \Delta \theta + \lambda(\Delta \theta^TH \Delta \theta - \epsilon^2) =: L(\Delta \theta, \lambda).
\end{equation*}
Then, $\Delta \theta^*(\lambda) = \mathrm{argmin}_{\Delta \theta} L(\Delta \theta, \lambda) = H^{-1}g_i/(2\lambda)$.
By replacing $\Delta \theta$ with $\Delta \theta^*(\lambda)$ in the dual problem, we get $\lambda^* = \mathrm{argmax}_{\lambda} L(\Delta \theta^*(\lambda), \lambda) = \sqrt{g_i^T H^{-1} g_i}/(2\epsilon)$.
Finally, we obtain the solution, $\Delta \theta^* = \Delta \theta^*(\lambda^*) = \epsilon H^{-1} g_i / \sqrt{g_i^T H^{-1} g_i}$.

Now let us find the solution of (\ref{eq: simplified converted problem}).
Under the linear assumption, we can rewrite (\ref{eq: simplified converted problem}) as follows:
\begin{equation}
\min_{\Delta \theta} \; \Delta \theta^T H \Delta \theta \;\; \mathrm{s.t.} \; e_i \leq g_i^T \Delta \theta.
\end{equation}
Similar to the above process, the Lagrange dual problem is derived as follows:
\begin{equation}
\max_{\lambda \geq 0} \min_{\Delta \theta} \; \Delta \theta^T H \Delta \theta + \lambda(e_i - g_i^T\Delta \theta) =: L(\Delta \theta, \lambda).
\end{equation}
Then, the solution of the dual problem is obtained as $\Delta \theta^*(\lambda) = \mathrm{argmin}_{\Delta \theta} L(\Delta \theta, \lambda) = \lambda H^{-1}g_i/2$.
Using $\Delta \theta^*(\lambda)$, we can get the optimal Lagrange multiplier as $\lambda^* = \mathrm{argmax}_{\lambda} L(\Delta \theta^*(\lambda), \lambda) = 2 e_i / (g_i H^{-1} g_i) = 2\epsilon/\sqrt{g_i H^{-1} g_i}$.
Then, the solution is $\Delta \theta^* = \Delta \theta^*(\lambda^*) = \epsilon H^{-1} g_i / \sqrt{g_i^T H^{-1} g_i}$.
As a result, the solutions of (\ref{eq: simplified problem}) and (\ref{eq: simplified converted problem}) are the same as $\epsilon H^{-1} g_i / \sqrt{g_i^T H^{-1} g_i}$.

\subsection{Proof of Theorem \ref{thm: generalized policy update rule}}
\label{sec: proof of generalized policy rule}


We first define the following notions for a policy $\pi$:
\begin{equation*}
\begin{aligned}
d_\rho^\pi(s) &:=(1-\gamma)\sum_{t=0}^\infty \gamma^t \mathrm{Pr}(s_t=s) \Big|  s_0 \sim \rho, a_t \sim \pi(\cdot|s_t), s_{t+1} \sim P(\cdot|s_t, a_t) \; \forall t, \\
V^\pi(s)&:= \mathbb{E} \left[ \sum_{t=0}^\infty \gamma^t R(s_t, a_t, s_{t+1}) \Big| s_0 = s, a_t \sim \pi(\cdot|s_t), s_{t+1} \sim P(\cdot|s_t, a_t)\; \forall t \right], \\
Q^\pi(s,a)&:= \mathbb{E} \left[ \sum_{t=0}^\infty \gamma^t R(s_t, a_t, s_{t+1}) \Big| s_0 = s, a_0=a, s_{t+1} \sim P(\cdot|s_t, a_t), a_{t+1} \sim \pi(\cdot|s_{t+1}) \; \forall t \right], \\
A^\pi(s,a) &:= Q^\pi(s,a) - V^\pi(s). \\
\end{aligned}
\end{equation*}
We also simplify several notations for brevity as follows:
\begin{equation*}
\pi_{\theta_{t}} \rightarrow \pi_t, \; d_\rho^{\pi_{\theta_t}} \rightarrow d_\rho^{t}, \; A_{R_i}^{\pi_{\theta_t}}(s,a) \rightarrow A_{R_i}^t(s,a), \; A_{C_k}^{\pi_{\theta_t}}(s,a) \rightarrow A_{C_k}^t(s,a).
\end{equation*}
Since it is assumed that the state and action spaces are finite, a policy can be parameterized using a softmax parameterization as follows \citep{agarwal2021theory}:
\begin{equation*}
\pi_\theta(a|s) := \exp(\theta(s, a)) \Big/ \sum_{a'\in A} \exp(\theta(s, a')),
\end{equation*}
where $\theta \in \mathbb{R}^{S\times A}$ is a trainable parameter, and $\theta(s,a)$ denotes the value corresponding to state $s$ and action $a$.
Through Lemma 4 in \citep{xu2021crpo}, the following equations are satisfied when a policy is updated according to (\ref{eq: generalized policy update rule}):
\begin{equation*}
\begin{aligned}
&\theta_{t+1} = \theta_{t}(s, a) + \frac{\alpha_t}{1-\gamma} \cdot \begin{cases}
(\sum_{i=1}^N \nu_{t,i}^a A_{R_i}^t - \alpha_t \sum_{k=1}^M \lambda_{t,k}^a A_{C_k}^t) & \text{if} \; J_{C_k}(\theta_t) \leq d_k \; \forall k,\\
(\alpha_t\sum_{i=1}^N \nu_{t,i}^b A_{R_i}^t - \sum_{k=1}^M \lambda_{t,k}^b A_{C_k}^t) & \text{otherwise},
\end{cases}\\
&\pi_{t+1}(a|s) = \pi_{t}(a|s)\frac{\exp(\theta_{t+1}(s, a) - \theta_{t}(s, a))}{Z_t(s)}, \\
&\text{where} \; Z_t(s):= \sum_{a\in A} \pi_{t}(a|s) \exp(\theta_{t+1}(s, a) - \theta_{t}(s, a)).
\end{aligned}
\end{equation*}
Now, we introduce a lemma showing that the LS approach guarantees to converge a CP optimal policy, followed by the proof of Theorem \ref{thm: convergence of modified CoMOGA}.
\begin{lemma}
\label{lemma: LS is CP}
Given a preference $\omega$, let us define an LS optimal policy as follows:
\begin{equation*}
\pi_\omega^* := \underset{\pi}{\arg \max} \sum_{i=1}^N \omega_i J_{R_i}(\pi) \quad \mathbf{s.t.} \; J_{C_k}(\pi) \leq d_k \; \forall k.
\end{equation*}
Then, the LS optimal policy is also a CP optimal policy.
\end{lemma}
\begin{proof}
Let us assume that there exists a policy $\mu$ that dominates the LS optimal policy $\pi_\omega^*$.
By definition of the constrained dominance, the following is satisfied:
\begin{equation*}
J_{R_i}(\mu) > J_{R_i}(\pi_\omega^*) \; \forall i \; \text{and} \; J_{C_k}(\mu) \leq d_k \; \forall k. \Rightarrow \sum_{i=1}^N \omega_i (J_{R_i}(\mu) - J_{R_i}(\pi_\omega^*)) > 0.
\end{equation*}
However, this equation contradicts the fact that $\pi_\omega^*$ is an LS optimal policy.
Consequently, $\pi_\omega^*$ is not dominated by any others, it is also a CP optimal policy.
\end{proof}

\GeneralPolicyUpdate*
\begin{proof}
We first consider the case $J_{C_k}(\theta_{t}) \leq d_k \; \forall k$, in which the policy is updated as follows:
\begin{equation*}
\theta_{t+1} = \theta_{t} + \frac{\alpha_t}{1-\gamma}(\sum_{i=1}^N \nu_{t,i}^a A_{R_i}^t - \alpha_t \sum_{k=1}^M \lambda_{t,k}^a A_{C_k}^t).
\end{equation*}
Using Lemma 1 in \citep{schulman2015trpo}, we can derive the following inequality:
\begin{equation}
\label{eq: performance difference}
\begin{aligned}
&\sum_i \nu_{t, i}^a(J_{R_i}(\theta_{t+1}) - J_{R_i}(\theta_{t})) - \alpha_t \sum_k \lambda_{t, k}^a(J_{C_k}(\theta_{t+1}) - J_{C_k}(\theta_{t})) \\
&= \frac{1}{1-\gamma} \mathbb{E}_{s \sim d_\rho^{t+1}} \left[ \mathbb{E}_{a \sim \pi_{t+1}(\cdot|s)} \left[ \sum_{i=1}^N \nu_{t,i}^a A_{R_i}^t(s, a) - \alpha_t \sum_{k=1}^M \lambda_{t,k}^a A_{C_k}^t(s, a) \right] \right] \\
&= \frac{1}{\alpha_t} \mathbb{E}_{s \sim d_\rho^{t+1}} \left[ \mathbb{E}_{a \sim \pi_{t+1}(\cdot|s)} \left[ \theta_{t+1}(s, a) - \theta_t(s, a)\right] \right] \\
& = \frac{1}{\alpha_t} \mathbb{E}_{s \sim d_\rho^{t+1}} \left[ \mathbb{E}_{a \sim \pi_{t+1}(\cdot|s)} \left[\log\frac{\pi_{t+1}(a|s) Z_t(s)}{\pi_t(a|s)} \right] \right] \\
& = \frac{1}{\alpha_t} \mathbb{E}_{s \sim d_\rho^{t+1}} \left[ D_\mathrm{KL}(\pi_{t+1}(\cdot|s)||\pi_t(\cdot|s)) + \log Z_t(s) \right] \\
&\geq \frac{1}{\alpha_t} \mathbb{E}_{s \sim d_\rho^{t+1}} \left[ \log Z_t(s) \right] \overset{\text{(i)}}{\geq} \frac{1-\gamma}{\alpha_t} \mathbb{E}_{s \sim \rho} \left[ \log Z_t(s) \right],
\end{aligned}
\end{equation}
where (i) follows from the fact that $d_\rho^\pi = (1-\gamma) \sum_{t=0}^\infty \gamma^t \mathrm{Pr}(s_t =s) \geq (1-\gamma)\rho(s)$.
By Lemma 5 in \citep{xu2021crpo},
\begin{equation*}
\begin{aligned}
|J_{R_i}(\theta_{t+1}) - J_{R_i}(\theta_{t})| &\leq \frac{2 R_\mathrm{max}}{1-\gamma}||\theta_{t+1} - \theta_t||_2 \; \forall i, \\
|J_{C_k}(\theta_{t+1}) - J_{C_k}(\theta_{t})| &\leq \frac{2 R_\mathrm{max}}{1-\gamma}||\theta_{t+1} - \theta_t||_2 \; \forall k. \\
\end{aligned}
\end{equation*}
\begin{equation*}
\begin{aligned}
\Rightarrow \; &\Big| \sum_i \nu_{t, i}^a(J_{R_i}(\theta_{t+1}) - J_{R_i}(\theta_{t})) - \alpha_t \sum_k \lambda_{t, k}^a(J_{C_k}(\theta_{t+1}) - J_{C_k}(\theta_{t})) \Big| \\
& \leq \frac{2R_\mathrm{max}}{1-\gamma}(N + \alpha_t M)\lambda_{\max}||\theta_{t+1} - \theta_t||_2 \\
&= \frac{2R_\mathrm{max}}{1-\gamma}(N + \alpha_t M)\lambda_{\max} \Big|\Big| \frac{\alpha_t}{1-\gamma}(\sum_{i=1}^N \nu_{t,i}^a A_{R_i}^t - \alpha_t \sum_{k=1}^M \lambda_{t,k}^a A_{C_k}^t) \Big|\Big|_2 \\
&\leq \frac{2\alpha_t R^2_\mathrm{max}}{(1-\gamma)^3}(N + \alpha_t M)^2\lambda_{\max}^2 \sqrt{|S| |A|}.
\end{aligned}
\end{equation*}
Then, (\ref{eq: performance difference}) can be rewritten as follows:
\begin{equation}
\label{eq: logZ inequality}
\frac{2\alpha_t R^2_\mathrm{max}}{(1-\gamma)^3}(N + \alpha_t M)^2\lambda_{\max}^2 \sqrt{|S| |A|} \geq \frac{1-\gamma}{\alpha_t} \mathbb{E}_{s \sim \rho} \left[ \log Z_t(s) \right],
\end{equation}
which is satisfied for any $\rho$.
Finally, we derive the following inequality for a policy $\mu$ as follows:
\begin{equation}
\label{eq: CS inequality}
\begin{aligned}
&\sum_i \nu_{t, i}^a(J_{R_i}(\mu) - J_{R_i}(\theta_{t})) - \alpha_t \sum_k \lambda_{t, k}^a(J_{C_k}(\mu) - J_{C_k}(\theta_{t})) \\
&= \frac{1}{1-\gamma} \mathbb{E}_{s \sim d_\rho^\mu} \left[ \mathbb{E}_{a \sim \mu(\cdot|s)} \left[ \sum_{i=1}^N \nu_{t,i}^a A_{R_i}^t(s, a) - \alpha_t \sum_{k=1}^M \lambda_{t,k}^a A_{C_k}^t(s, a) \right] \right] \\
& = \frac{1}{\alpha_t} \mathbb{E}_{s \sim d_\rho^\mu} \left[ \mathbb{E}_{a \sim \mu(\cdot|s)} \left[\log\frac{\pi_{t+1}(a|s) Z_t(s)}{\pi_t(a|s)} \right] \right] \\
& = \frac{1}{\alpha_t} \mathbb{E}_{s \sim d_\rho^\mu} \left[ D_\mathrm{KL}(\mu(\cdot|s)||\pi_t(\cdot|s)) - D_\mathrm{KL}(\mu(\cdot|s)||\pi_{t+1}(\cdot|s)) + \log Z_t(s) \right] \\
&\overset{\text{(i)}}{\leq} \frac{1}{\alpha_t} \mathbb{E}_{s \sim d_\rho^\mu} \left[ D_\mathrm{KL}(\mu(\cdot|s)||\pi_t(\cdot|s)) - D_\mathrm{KL}(\mu(\cdot|s)||\pi_{t+1}(\cdot|s)) \right] \\
& \quad + \frac{2\alpha_t R^2_\mathrm{max}}{(1-\gamma)^3}(N + \alpha_t M)^2\lambda_{\max}^2 \sqrt{|S| |A|}, \\
\end{aligned}
\end{equation}
where (i) is from (\ref{eq: logZ inequality}) by replacing $\rho$ with $d_\rho^\mu$.
For brevity, we denote $\mathbb{E}_{s \sim d_\rho^\mu} \left[ D_\mathrm{KL}(\mu(\cdot|s)||\pi_t(\cdot|s)) \right]$ as $D_\mathrm{KL}(\mu||\pi_t)$.
Now, we consider the second case where the constraints are violated.
In this case,
\begin{equation*}
\theta_{t+1} = \theta_{t} + \frac{\alpha_t}{1-\gamma}(\alpha_t \sum_{i=1}^N \nu_{t,i}^b A_{R_i}^t - \sum_{k=1}^M \lambda_{t,k}^b A_{C_k}^t),
\end{equation*}
which is symmetrical to the first case, achieved by replacing $\nu_{t,i}^a \rightarrow \alpha_t \nu_{t,i}^b$ and $\alpha_t\lambda_{t,k}^a \rightarrow \lambda_{t,k}^b$.
Using this symmetry property, we can obtain the following inequality from (\ref{eq: CS inequality}):
\begin{equation}
\label{eq: CV inequality}
\begin{aligned}
&\alpha_t \sum_i \nu_{t, i}^b(J_{R_i}(\mu) - J_{R_i}(\theta_{t})) - \sum_k \lambda_{t, k}^b(J_{C_k}(\mu) - J_{C_k}(\theta_{t})) \\
&\leq \frac{1}{\alpha_t}( D_\mathrm{KL}(\mu||\pi_t) - D_\mathrm{KL}(\mu||\pi_{t+1})) + \frac{2\alpha_t R^2_\mathrm{max}}{(1-\gamma)^3}(\alpha_t N + M)^2\lambda_{\max}^2 \sqrt{|S| |A|}. \\
\end{aligned}
\end{equation}
Now, we define a set of time steps, $\mathcal{N}:= \{t| J_{C_k}(\pi_t) \leq d_k \; \forall k\}$.
Then, by summing (\ref{eq: CS inequality}) and (\ref{eq: CV inequality}) over several time steps, the following is satisfied for a policy $\mu$:
\begin{equation}
\label{eq: summation of all time step}
\begin{aligned}
&\sum_{t\in\mathcal{N}} \left( \alpha_t \sum_i \nu_{t, i}^a(J_{R_i}(\mu) - J_{R_i}(\theta_{t})) - \alpha_t^2 \sum_k \lambda_{t, k}^a(J_{C_k}(\mu) - J_{C_k}(\theta_{t})) \right) \\
&+ \sum_{t\notin\mathcal{N}} \left( \alpha_t^2 \sum_i \nu_{t, i}^b(J_{R_i}(\mu) - J_{R_i}(\theta_{t})) - \alpha_t \sum_k \lambda_{t, k}^b(J_{C_k}(\mu) - J_{C_k}(\theta_{t})) \right) \\
&\leq D_\mathrm{KL}(\mu||\pi_0) + \sum_t \alpha_t^2 \underbrace{\frac{2 R^2_\mathrm{max}}{(1-\gamma)^3}\max(\alpha_0, 1)^2(N + M)^2\lambda_{\max}^2 \sqrt{|S| |A|}}_{=: \kappa_1}.
\end{aligned}
\end{equation}
If $t \notin \mathcal{N}$, $\lambda_{t,k}^b = 0$ for $J_{C_k}(\theta_t) \leq d_k$, which results in $\sum_k \lambda_{t,k}^b (J_{C_k}(\pi_f) - J_{C_k}(\theta_t)) \leq -\eta$.
By using this fact and replacing $\mu$ in (\ref{eq: summation of all time step}) with $\pi_f$,
\begin{equation*}
\begin{aligned}
&\sum_{t\in\mathcal{N}} \left( \alpha_t \sum_i \nu_{t, i}^a(J_{R_i}(\pi_f) - J_{R_i}(\theta_{t})) - \alpha_t^2 \sum_k \lambda_{t, k}^a(J_{C_k}(\pi_f) - J_{C_k}(\theta_{t})) \right) \\
&+ \sum_{t\notin\mathcal{N}} \left( \alpha_t^2 \sum_i \nu_{t, i}^b(J_{R_i}(\pi_f) - J_{R_i}(\theta_{t})) + \alpha_t\eta \right) \leq D_\mathrm{KL}(\pi_f||\pi_0) + \kappa_1 \sum_t \alpha_t^2.
\end{aligned}
\end{equation*}
\begin{equation*}
\begin{aligned}
&\Rightarrow \sum_{t\in\mathcal{N}} \left( \alpha_t \sum_i \nu_{t, i}^a(J_{R_i}(\pi_f) - J_{R_i}(\theta_{t})) \right) + \sum_{t\notin\mathcal{N}} \alpha_t\eta \\
&\leq D_\mathrm{KL}(\pi_f||\pi_0) + \kappa_1 \sum_t \alpha_t^2 + \sum_{t\in\mathcal{N}}\alpha_t^2 \sum_k \lambda_{t, k}^a(J_{C_k}(\pi_f) - J_{C_k}(\theta_{t})) \\
&\quad - \sum_{t\in\mathcal{N}}\alpha_t^2 \sum_i \nu_{t, i}^b(J_{R_i}(\pi_f) - J_{R_i}(\theta_{t})) \\
&\leq D_\mathrm{KL}(\pi_f||\pi_0) + \sum_t \alpha_t^2\underbrace{\left(\kappa_1 + \frac{2R_\mathrm{max} \lambda_\mathrm{max}}{1-\gamma} \max(N, M) \right)}_{=: \kappa_2}.
\end{aligned}
\end{equation*}
Since $\sum_{t\notin \mathcal{N}}\alpha_t \eta = \sum_t \alpha_t \eta - \sum_{t\in \mathcal{N}}\alpha_t \eta$,
\begin{equation*}
\begin{aligned}
&\sum_{t\in\mathcal{N}} \left( \alpha_t \left(\sum_i \nu_{t, i}^a(J_{R_i}(\pi_f) - J_{R_i}(\theta_{t})) - \eta \right) \right) + \sum_{t} \alpha_t\eta \leq D_\mathrm{KL}(\pi_f||\pi_0) + \sum_t \alpha_t^2 \kappa_2.
\end{aligned}
\end{equation*}
Due to the Robbins-Monro condition, the right term converges to real number.
Since $\sum_t \alpha_t \eta = \infty$ in the left term, the following must be satisfied:
\begin{equation*}
\sum_{t\in\mathcal{N}} \left( \alpha_t \left(\sum_i \nu_{t, i}^a(J_{R_i}(\pi_f) - J_{R_i}(\theta_{t})) - \eta \right) \right) = -\infty,
\end{equation*}
which results in $\sum_{t\in\mathcal{N}} \alpha_t = \infty$.
Now, let us define a preference $\omega$, where $\omega_i:= \bar{\nu}^a_i$.
By replacing $\mu$ in (\ref{eq: summation of all time step}) with $\pi_\omega^*$,
\begin{equation*}
\begin{aligned}
&\sum_{t\in\mathcal{N}} \left( \alpha_t \sum_i \nu_{t, i}^a(J_{R_i}(\pi_\omega^*) - J_{R_i}(\theta_{t})) \right) - \sum_{t\notin\mathcal{N}} \left( \alpha_t \sum_k \lambda_{t, k}^b(J_{C_k}(\pi_\omega^*) - J_{C_k}(\theta_{t})) \right) \\
&\leq D_\mathrm{KL}(\pi_\omega^*||\pi_0) + \kappa_2 \sum_t \alpha_t^2.
\end{aligned}
\end{equation*}
\begin{equation*}
\begin{aligned}
&\Rightarrow \sum_{t\in\mathcal{N}} \left( \alpha_t \sum_i \nu_{t, i}^a(J_{R_i}(\pi_\omega^*) - J_{R_i}(\theta_{t})) \right) \leq D_\mathrm{KL}(\pi_\omega^*||\pi_0) + \kappa_2 \sum_t \alpha_t^2.
\end{aligned}
\end{equation*}
Since the right term converges to a real number and $\sum_{t\in \mathcal{N}}\alpha_t = \infty$, the following must be satisfied:
\begin{equation*}
\lim_{t\to \infty}\sum_i \nu_{t, i}^a(J_{R_i}(\pi_\omega^*) - J_{R_i}(\theta_{\mathcal{N}_t})) = 0,
\end{equation*}
where $\mathcal{N}_t$ is the $t$th element of $\mathcal{N}$.
Consequently,
\begin{equation*}
\sum_i \bar{\nu}_i^a J_{R_i}(\pi_\omega^*) = \sum_i \bar{\nu}_i^a \lim_{t\to \infty}J_{R_i}(\pi_{\mathcal{N}_t}),
\end{equation*}
which means that $\lim_{t\to \infty} \pi_{\mathcal{N}_t}$ is also a LS optimal policy for the preference $\omega$.
Due to Lemma \ref{lemma: LS is CP}, the policy converges to a CP optimal policy.
\end{proof}

\subsection{Proof of Theorem \ref{thm: convergence of modified CoMOGA}}
\label{sec: proof of CoMOGA}

\CoMOGA*
\begin{proof}
This theorem can be proved by identifying sequences $\nu_{t,i}^a, \nu_{t,i}^b, \lambda_{t,k}^a, \lambda_{t,k}^b$ that satisfy all conditions mentioned in Theorem \ref{thm: generalized policy update rule}.
By examining (\ref{eq: modified CoMOGA}), $\nu_{t,i}^a$, $\nu_{t,i}^b$, $\lambda_{t,k}^a$, $\lambda_{t,k}^b$, and $\alpha_t$ can be deduced as follows:
\begin{equation*}
\begin{aligned}
&\nu_{t,i}^a = \frac{\nu^*_i}{\sum_{j=1}^N \nu^*_j}, 
\nu_{t,i}^b = 0, 
\lambda_{t,k}^b = \frac{\lambda^*_k}{\sum_{j=1}^M \lambda^*_j}, \\
&\lambda_{t,k}^a = \min(\max(||\bar{g}_\omega^\mathrm{ag}||, g_\mathrm{min}), g_\mathrm{max}) \cdot \min \left(\frac{\lambda_k^*}{\epsilon_t \sum_{i=1}^N\nu_i^*}, \lambda_\mathrm{max} \right), \\
&\alpha_t = \frac{\epsilon_t}{\min(\max(||\bar{g}_\omega^\mathrm{ag}||, g_\mathrm{min}), g_\mathrm{max})}.
\end{aligned}
\end{equation*}
Substituting the deduced sequences into (\ref{eq: modified CoMOGA}), we can obtain the same formulation as (\ref{eq: generalized policy update rule}).
The deduced sequences satisfy that $\sum \nu_{t,i}^a = 1$, $\sum \lambda_{t,k}^b = 1$, $\nu_{t,i}^b \in \mathbb{R}_{\geq 0}$ is bounded, $0 \leq \lambda_{t,k}^a \leq g_\mathrm{max} \lambda_\mathrm{max}$ is also bounded, and $\alpha_t \in [\epsilon_t/g_\mathrm{max}, \epsilon_t/g_\mathrm{min}]$ follows the Robbins-Monro condition.
Additionally, since $\lambda_k^*$ is an optimal point of the dual problem, it satisfies the KKT condition, which results in $\lambda_k^*=0$ if $J_{C_k}(\theta_t) < d_k$.
Consequently, $\lambda_{t,k}^b (J_{C_k}(\theta_t) - d_k) \geq 0$.
Now, we need to check whether the $\nu_{t,i}^a$ converges to a real number.
Given that the learning rate follows the Robbins-Monro condition, the policy converges to a specific policy. 
Since (\ref{eq: gradient aggregation}) becomes invariant at this point of convergence, the solution of the dual problem is also fixed.
Consequently, $\nu_{t,i}^a$ converge to a specific value.
\end{proof}

\newpage
\section{Experimental Details}
\label{sec: experiment detail}

In this section, we describe the details of each task, including Safety Gymnasium \citep{ji2023safetygym}, legged robot locomotion \citep{kim2023sdac}, and Multi-Objective Gymnasium \citep{felten2023mogym}. We then provide a detailed explanation of the performance metrics named hypervolume and sparsity. 
Finally, we present implementation details, such as network architectures and hyperparameter settings.

\subsection{Task Details}

\subsubsection{Safety Gymnasium}

\begin{figure}[h]
\centering
\begin{subfigure}[b]{0.23\textwidth}
    \centering
    \includegraphics[width=1.0\textwidth]{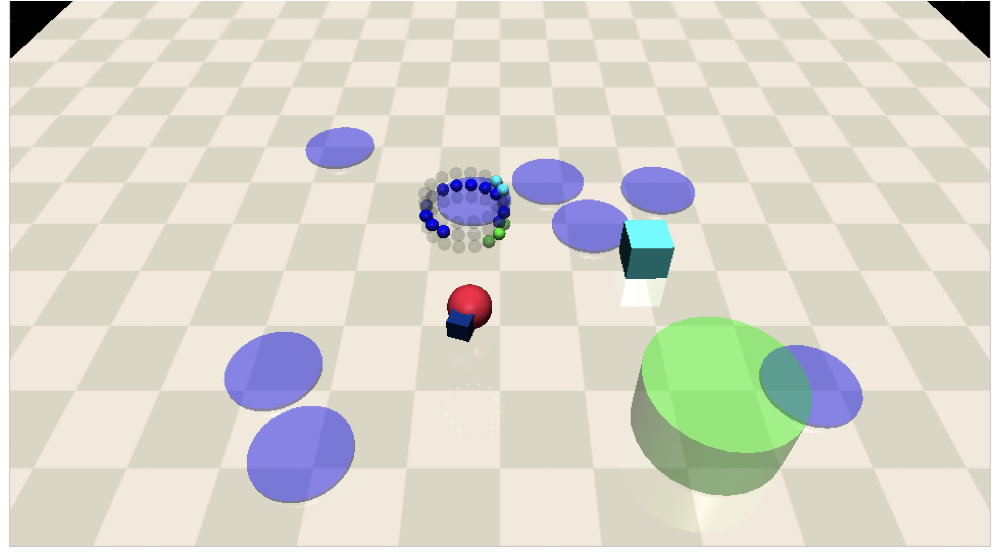}
    \caption{Single-Point-Goal}
\end{subfigure}
\hfill
\begin{subfigure}[b]{0.23\textwidth}
    \centering
    \includegraphics[width=1.0\textwidth]{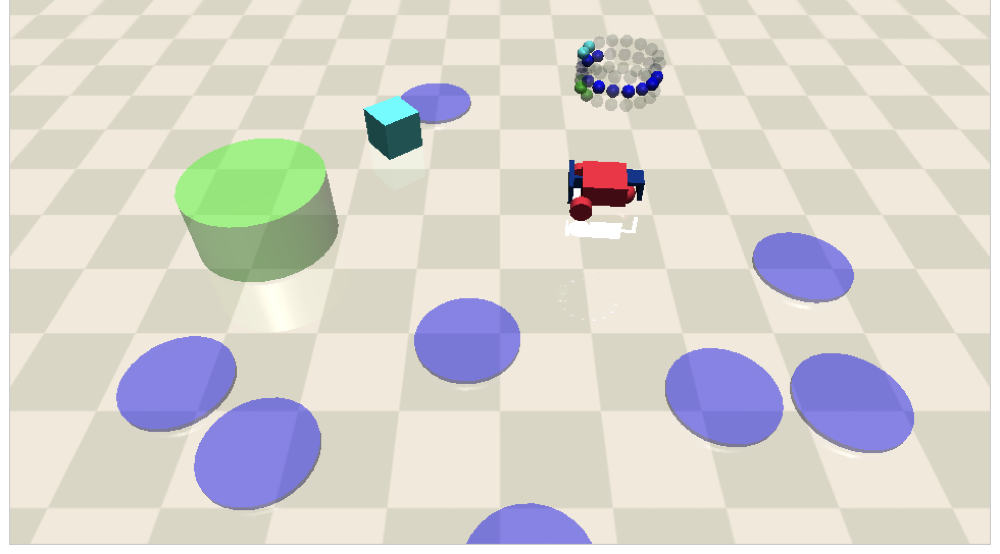}
    \caption{Single-Car-Goal}
\end{subfigure}
\hfill
\begin{subfigure}[b]{0.23\textwidth}
    \centering
    \includegraphics[width=1.0\textwidth]{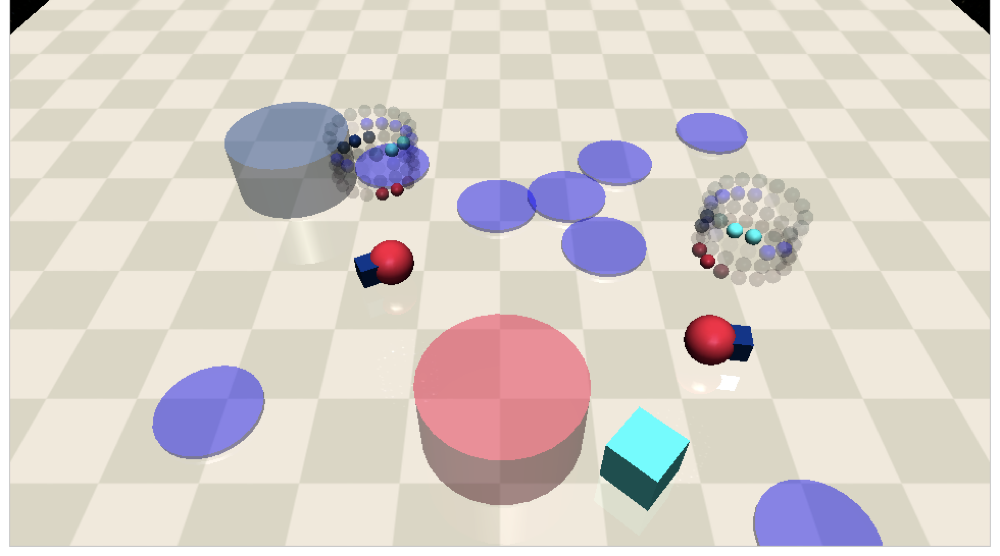}
    \caption{Multi-Point-Goal}
\end{subfigure}
\hfill
\begin{subfigure}[b]{0.23\textwidth}
    \centering
    \includegraphics[width=1.0\textwidth]{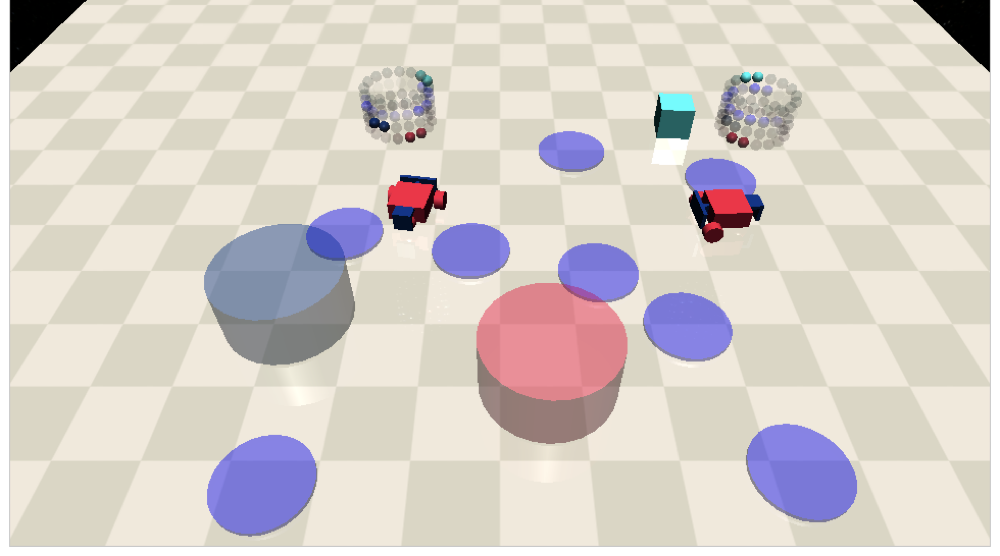}
    \caption{Multi-Car-Goal}
\end{subfigure}
\caption{\textbf{Snapshots of the Safety-Gymnasium tasks.}
There are two types of robots: point and car.
These tasks aim to control the robots to reach goals while avoiding hazardous areas, colored purple.
In the Single-Point-Goal and Single-Car-Goal tasks, a single goal, shown in green, is randomly spawned. 
Conversely, the Multi-Point-Goal and Multi-Car-Goal tasks have two goals, shown in blue and red, and if an agent reaches any of them, a reward is given to the agent. 
}
\label{fig: snapshot of safety gym}
\end{figure}

In the Safety Gymnasium \citep{ji2023safetygym}, we utilize the single-agent safe navigation tasks: \texttt{SafetyPointGoal1-v0} and \texttt{SafetyCarGoal1-v0}, and the multi-agent safe navigation tasks: \texttt{SafetyPointMultiGoal1-v0} and \texttt{SafetyCarMultiGoal1-v0}.
Snapshots of each task are shown in Figure \ref{fig: snapshot of safety gym}. 

In the single-agent tasks, the observation space includes velocity, acceleration, and Lidar of the goal and hazards, and the dimensions are 60 and 72 for the point and car robots, respectively.
The action space for both robots is two-dimensional. 
There are eight hazard areas and one goal, and hazard areas are randomly spawned at the beginning of each episode, while goals are randomly placed whenever the agent reaches the current goal.
These tasks originally had a single objective: maximizing the number of goals reached. 
However, in order to have multiple objectives, we have modified them to include a second objective: minimizing energy consumption. 
Therefore, the reward is two-dimensional, and we use the original definition for the first reward (please refer to \citet{ji2023safetygym}).
The second reward is defined as follows:
\begin{equation}
R_2(s, a, s') := -\frac{1}{|A|}\sum_{i=1}^{|A|} (a_i/10)^2.
\end{equation}
Also, we use the original definition of the cost function, which gives one if the agent enters hazardous areas and zero otherwise.

The multi-agent tasks have two agents and two goals, and we need to control both agents to reach goals as much as possible while avoiding hazardous areas.
The original implementation uses a dictionary type for observations and actions, so we have modified them to be an array type.
Additionally, in the original setup, each goal is pre-assigned to a specific agent, reducing the difficulty of solving multi-agent tasks.
Therefore, we have modified the implementation so that each goal is not pre-assigned, allowing each agent to compete to achieve goals.
The observation space contains the same information from the single-agent tasks for the two agents, as well as additional information: Lidar of the first and second goals.
The dimensions of observation space are then 152 and 176 for the point and car robots, respectively.
The action space is four-dimensional, which is doubled from the single-agent tasks to control both agents.
These tasks have two objectives: maximizing the number of goals reached by each agent. 
Also, there are two constraints: avoiding hazardous areas for each agent.
Since there are no additional objectives or constraints, the original definitions of the reward and cost functions are used without modification.

\subsubsection{Legged Robot Locomotion}

\begin{figure}[ht]
\centering
\hspace{0.02\textwidth}%
\begin{subfigure}[b]{0.3\textwidth}
    \centering
    \includegraphics[width=1.0\textwidth]{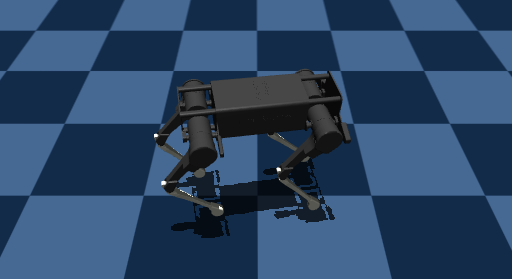}
    \caption{Quadruped}
\end{subfigure}
\hspace{0.15\textwidth}%
\begin{subfigure}[b]{0.3\textwidth}
    \centering
    \includegraphics[width=1.0\textwidth]{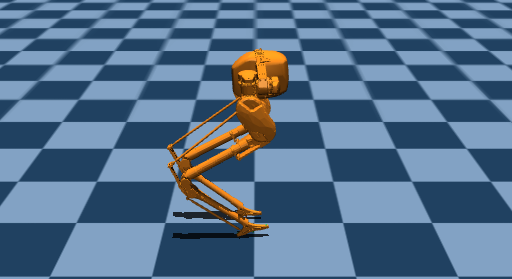}
    \caption{Biped}
\end{subfigure}
\hspace{0.05\textwidth}%
\caption{\textbf{Snapshots of the legged robot locomotion tasks.}
Robots aim to follow a given command while ensuring they do not fall over.}
\label{fig: snapshot of legged robot}
\vspace{-10pt}
\end{figure}

The legged robot locomotion tasks \citep{kim2023sdac} aim to control bipedal or quadrupedal robots so that their velocity matches a randomly sampled command. 
This command specifies the target linear velocities in the x and y-axis directions 
and the target angular velocity in the z-axis direction, denoted by ($v_x^\mathrm{cmd}, v_y^\mathrm{cmd}, \omega_z^\mathrm{cmd}$).
Snapshots of these tasks are shown in Figure \ref{fig: snapshot of legged robot}.
The quadrupedal robot has 12 joints and 12 motors, and the bipedal robot has 14 joints and 10 motors.
Each robot is operated by a PD controller that follows the target position of the motors, and the target position is given as an action.
Hence, the number of motors corresponds to the dimension of action space.
In order to provide enough information for stable control, observations include the command, linear and angular velocities of the robot base, and the position and velocity of each joint. 
As a result, the dimensions of the observation space are 160 for the quadruped and 132 for the biped.
Originally, these tasks have a single objective of following a given command and three constraints: 1) maintaining body balance, 2) keeping the height of CoM (center of mass) above a certain level, and 3) adhering to pre-defined foot contact timing.
Therefore, we use the original implementation for the observation, action, and cost functions but modify the reward function to have multiple objectives.
The modified version has two objectives: 1) reducing the difference between the current velocity and the command and 2) minimizing energy consumption.
Then, the reward function is defined as follows:
\begin{equation}
\begin{aligned}
&R_1(s, a, s') := 1 - (v_x^\mathrm{base} - v_x^\mathrm{cmd})^2 - (v_y^\mathrm{base} - v_y^\mathrm{cmd})^2 - (\omega_z^\mathrm{base} - \omega_z^\mathrm{cmd})^2, \\
&R_2(s, a, s') := 1 - \frac{1}{J}\sum_{j=1}^J \left(\frac{\tau_j}{M_\mathrm{robot}}\right)^2,
\end{aligned}
\end{equation}
where $J$ is the number of joints, $\tau_j$ is the torque applied to the $j$th joint, and $M_\mathrm{robot}$ is the mass of the robot.
However, in the simulation of the bipedal robot, obtaining valid torque information is difficult due to the presence of closed loops in the joint configuration.
To address this issue, we modify the second reward for the bipedal robot to penalize the action instead of the joint torque as follows:
\begin{equation}
R_2(s,a,s') := 1 - \frac{1}{|A|} \sum_{i=1}^{|A|} a_i^2.
\end{equation}

\subsubsection{Multi-Objective Gymnasium}

\begin{table}[h]
    \caption{Specifications of MO-Gymnasium tasks.}
    \label{tab: description of mo gym}
    \centering
    \begin{tabular}{c|ccc}
        \toprule
         &  Observation Space&  Action Space&  \texttt{\#} of Objectives (Entries)\\
        \midrule
         Half-Cheetah&  $\subseteq\mathbb{R}^{17}$&  $\subseteq\mathbb{R}^{6}$& 2 (velocity, energy)\\
         Hopper&  $\subseteq\mathbb{R}^{11}$&  $\subseteq\mathbb{R}^{3}$& 3 (velocity, height, energy)\\
         Humanoid&  $\subseteq\mathbb{R}^{376}$&  $\subseteq\mathbb{R}^{17}$& 2 (velocity, energy)\\
         Ant&  $\subseteq\mathbb{R}^{27}$&  $\subseteq\mathbb{R}^{8}$& 3 ($x$-velocity, $y$-velocity, energy)\\
         Walker2d&  $\subseteq\mathbb{R}^{17}$&  $\subseteq\mathbb{R}^{6}$& 2 (velocity, energy)\\
         Swimmer&  $\subseteq\mathbb{R}^{8}$&  $\subseteq\mathbb{R}^{2}$& 2 (velocity, energy)\\
         \bottomrule
    \end{tabular}
\end{table}
We utilize the Multi-Objective (MO) Gymnasium \citep{felten2023mogym} to evaluate the proposed method in MORL tasks which have no constraints.
Among several tasks, we use the MuJoCo tasks with continuous action space, and there are six tasks available: \texttt{mo-hopper-v4}, \texttt{mo-humanoid-v4}, \texttt{mo-halfcheetah-v4}, \texttt{mo-ant-v4}, \texttt{mo-walker2d-v4}, and \texttt{mo-swimmer-v4}.
Description of the observation space, action space, and the number of objectives for these tasks are summarized in Table \ref{tab: description of mo gym}, and for more details, please refer to \citet{felten2023mogym}.

\subsection{Metric Details}
\label{sec: performance metric}

\begin{figure}[ht]
\centering
\begin{subfigure}[b]{0.3\textwidth}
    \centering
    \includegraphics[width=1.0\textwidth]{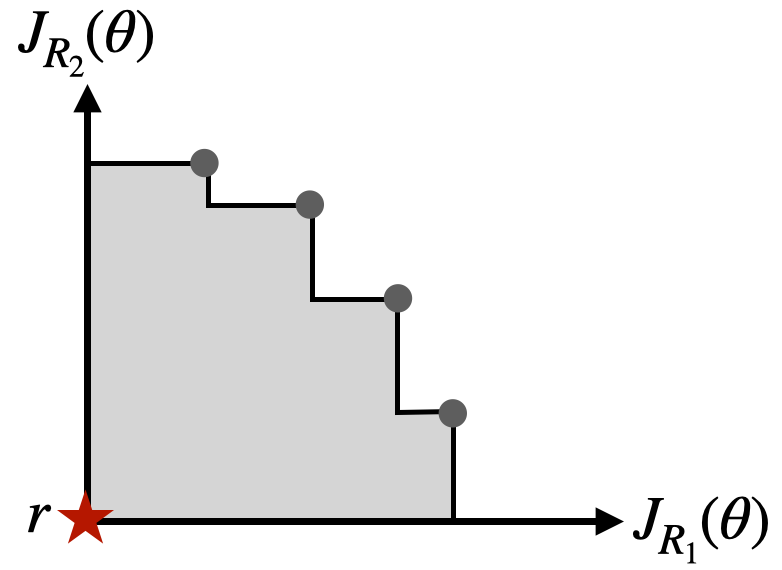}
    \caption{Hypervolume}
\end{subfigure}
\hfill
\begin{subfigure}[b]{0.3\textwidth}
    \centering
    \includegraphics[width=1.0\textwidth]{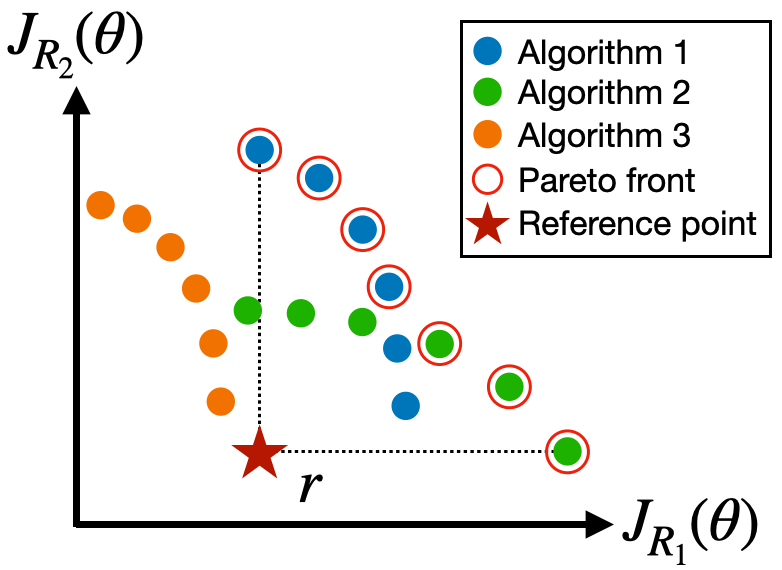}
    \caption{Reference Point}
\end{subfigure}
\hfill
\begin{subfigure}[b]{0.3\textwidth}
    \centering
    \includegraphics[width=1.0\textwidth]{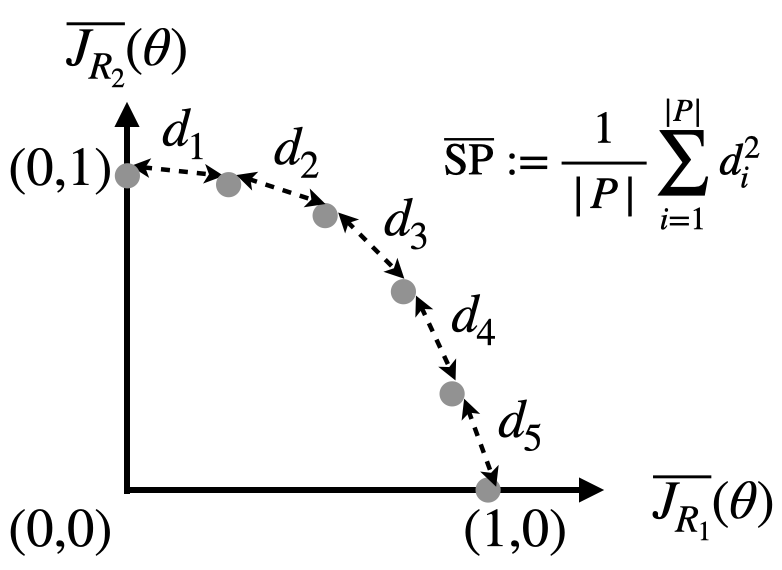}
    \caption{Normalized Sparsity}
\end{subfigure}
\caption{\textbf{Details of performance metrics.}
The hypervolume is represented by the gray area in (a). 
In (b), the red star indicates the reference point, determined from the entire Pareto front, whose elements are circled in red. 
In (c), the normalized sparsity is calculated by averaging the squared distances between elements in the normalized space.
}
\label{fig: performance metric}
\vspace{-10pt}
\end{figure}

In this section, we explain in detail the performance metrics used in the main text: hypervolume (HV) and normalized sparsity ($\overline{\mathrm{SP}}$).
First, we need to prepare an estimated CP front to calculate such metrics.
To do this, we pre-sample a fixed number of equally spaced preferences and roll out the trained universal policy for each pre-sampled preference to estimate objective and constraint values.
If the constraint values do not exceed the thresholds, we store the objective values, $(J_{R_1}(\theta), ..., J_{R_N}(\theta))$, in a set.
Once this process is completed for all preferences, we can construct an estimated CP front by extracting elements from the collected set that are not dominated by any others.
Now, we introduce details on each metric.

HV is the volume of the area surrounding a given reference point, $r$, and an estimated CP front, $P$, as defined in (\ref{eq: hypervolume}) and is visualized in Figure \ref{fig: performance metric}(a).
However, as seen in Figure \ref{fig: performance metric}(a) or (\ref{eq: hypervolume}), the metric value varies depending on what reference point is used.
Therefore, we do not set the reference point arbitrarily but use a method of determining the reference point from the results of all algorithms, as shown in Figure \ref{fig: performance metric}(b).
For a detailed explanation, suppose that we have a set of estimated CP fronts, $\{P_i\}_{i=1}^K$, where $P_i$ is the estimated front obtained from the $i$th algorithm.
Then, we obtain the union of $\{P_i\}_{i=1}^K$, denoted by $P'$, and extract elements that are not dominated by any others in $P'$ to find the Pareto front of the entire set, $\mathrm{PF}(P')$, whose elements are circled in red in Figure \ref{fig: performance metric}(b).
Finally, we get the reference point whose $i$th value is $\mathrm{min}(\{p_i| \forall p \in \mathrm{PF}(P')\})$.
Through this process, the reference point can be automatically set after obtaining results from all algorithms for each task.
In addition, the HV value can be zero for some algorithms whose elements are entirely dominated by $\mathrm{PF}(P')$.
An example of this is Algorithm 3 in Figure \ref{fig: performance metric}(b).

As mentioned in the main text, sparsity (SP) \citep{xu2020pgmorl} has a correlation with HV is defined as follows:
\begin{equation}
\label{eq: sparsity}
\mathrm{SP}(P) := \frac{1}{|P| - 1}\sum_{j=1}^{N}\sum_{i=1}^{|P| - 1} (\tilde{P}_j[i] - \tilde{P}_j[i+1])^2,
\end{equation}
where $\tilde{P}_j := \mathrm{Sort}(\{p[j]|\forall p \in P\})$.
An increase in HV implies that the elements of the CP front are moving further away from each other, which increases SP.
To remove this correlation, we provide a normalized version as defined in (\ref{eq: normalized sparsity}), and the calculation process is illustrated in Figure \ref{fig: performance metric}(c).
Given an estimated CP front, $P$, we normalize each estimated objective value so that the minimum value corresponds to zero and the maximum to one.
The normalized objectives are denoted as $\overline{J_{R_i}}(\theta)$.
We then calculate the average of the squares of the distances between elements in this normalized space. 
This normalized version still produces a large value when a CP front is densely clustered in particular areas while being sparsely distributed overall. 
Therefore, it maintains the ability to measure sparsity while removing the correlation with HV.

\subsection{Implementation Details}
\label{sec: implementation details}

In this section, several implementation details including hyperparameter settings and network structures are provided.
Before that, we report information on the computational resources.
In all experiments, we used a PC equipped with an Intel Xeon CPU E5-2680 and an NVIDIA TITAN Xp GPU.
The average training time in the single point goal task is presented in Table \ref{tab: computation time}.

\begin{table}[h]
    \caption{Training time on the point single goal task.}
    \label{tab: computation time}
    \centering
    \begin{tabular}{c|cccc}
        \toprule
           &  CAPQL&  PD-MORL & LP3 & CoMOGA (ours) \\
        \midrule
        Time & 16h 49m 29s& 16h 40m 9s& 24h 4m 34s& 17h 27m 58s \\
        \bottomrule
    \end{tabular}
\end{table}

\subsubsection{MORL to CMORL using the Lagrangian method}
\label{sec: ls}

In order to apply linear scalarization-based MORL methods to CMORL, it is required to solve the following constrained optimization problem given a preference $\omega$:
\begin{equation}
\mathrm{max}_\theta \sum_{i=1}^N \omega_i J_{R_i}(\theta) \; \mathrm{s.t.} \; J_{C_k}(\theta) \leq d_k \; \forall k \in \{1, .., M\}.
\end{equation}
As explained in the main text, the constraints of the above problem can be handled by converting it into a Lagrange dual problem, and the dual problem is written as follows:
\begin{equation}
\mathrm{max}_{\lambda \geq 0} \mathrm{min}_\theta -\sum_{i=1}^N \omega_i J_{R_i}(\theta) + \sum_{k=1}^M \lambda_k \cdot (J_{C_k}(\theta) - d_k),
\end{equation}
where $\lambda_k$ are Lagrange multipliers.
The multiplies should be learned separately for each preference, but preferences have continuous values. 
Therefore, we instead parameterize the multiplier using neural networks as $\lambda_\phi(\omega)$, where $\phi$ is a parameter.
Then, the above problem can be solved by concurrently updating the policy parameter and the Lagrange multipliers using the following loss functions:
\begin{equation}
\label{eq: Lagrangian method}
\begin{aligned}
\mathrm{min}_\theta L(\theta) &:= -\sum_{i=1}^N \omega_i J_{R_i}(\theta) + \sum_{k=1}^M \lambda_k J_{C_k}(\theta), \\
\mathrm{min}_{\phi} L(\phi) &:= -\sum_{k=1}^M \lambda_{k, \phi}(\omega)(J_{C_k}(\theta) - d_k) \; \mathrm{s.t.} \; \lambda_{k, \phi}(\omega) \geq 0,
\end{aligned}
\end{equation}
where the multipliers can be forced to non-negative values using some activation functions, such as softplus and exponential functions.
The policy loss in (\ref{eq: Lagrangian method}) is constructed by adding $\sum \lambda_k J_{C_k}(\theta)$ to the loss from the unconstrained MORL problem. 
Therefore, the implementation of the existing MORL algorithms can be easily extended to a CMORL algorithm by adding $\sum \lambda_k J_{C_k}(\theta)$ to the original policy loss.

\subsubsection{Network Architecture}

The proposed method and baselines are based on the actor-critic framework, requiring policy and critic networks.
Additionally, since the baselines use the Lagrangian method to handle the constraints, they use multiplier networks, as mentioned in Appendix \ref{sec: ls}.
We have implemented these networks as fully connected networks (FCNs), and their structures are presented in Table \ref{table: network structure}.
\begin{table}[ht]
\caption{Network structures.}
\label{table: network structure}
\centering
\resizebox{0.5\textwidth}{!}{
\begin{tabular}{c|c|c}
\toprule
 & Parameter& Value \\
\midrule
 Policy network&Hidden layer& (512, 512)\\
 &Activation& LeakyReLU\\
 & Normalization&\ding{56}\\
\midrule
 Critic network&Hidden layer& (512, 512)\\
 &Activation& LeakyReLU\\
 & Normalization&\ding{56}\\
\midrule
 Quantile distributional& Hidden layer&(512, 512)\\
critic network & Activation&LeakyReLU\\
 & Normalization&LayerNorm\\
 & \texttt{\#} of quantiles&25\\
\midrule
 Multiplier network& Hidden layer&(512,)\\
& Activation&LeakyReLU\\
 & Normalization&\ding{56}\\
\bottomrule
\end{tabular}
}
\end{table}
In the unconstrained tasks, MO-Gymnasium, we use standard critic networks that output scalar values for given observations and preferences.
However, the standard critic networks usually estimate the objective or constraint functions with large biases, making it challenging to satisfy the constraints.
Therefore, we use quantile distributional critics \citep{dabney2018quantile} to lower the estimation biases in the Safety-Gymnasium and legged robot locomotion experiments.
The quantile distribution critic outputs several quantile values of the discounted sum of the reward or cost functions for given observations and preferences.
Then, the objective or constraint functions can be estimated by averaging the outputted quantiles.
For the proposed method and LP3 that have similar frameworks to TRPO \citep{schulman2015trpo}, TD($\lambda$) method \citep{kim2023sdac} is used to train the distributional critic networks.
For PD-MORL and CAPQL, the distributional critic networks are trained by reducing the Wasserstein distance between the current quantiles and the truncated one-step TD targets, as in TQC \citep{kuznetsov2020tqc}.
The policy network, similar to other RL algorithms dealing with continuous action spaces \citep{schulman2015trpo, haarnoja2018sac, fujimoto2018td3}, outputs the mean and diagonal variance of a normal distribution.
However, in the locomotion tasks, the policy networks are modified to output only the mean value by fixing the diagonal variance in order to lower the training difficulty.

\subsubsection{Hyperparameter Settings}
\label{sec: hyperparameter settings}

We report the hyperparameter settings for the CMORL tasks (Safety-Gymnasium, Locomotion) in Table \ref{tab: hyperparameter for CMORL} and the settings for the MORL tasks (MO-Gymnasium) in Table \ref{tab: hyperparameter for MORL}.
\begin{table}[bht]
    \caption{Hyperparameters for CMORL tasks.}
    \label{tab: hyperparameter for CMORL}
    \centering
    \resizebox{0.95\textwidth}{!}{
        \begin{tabular}{c|cccc}
        \toprule
                & CAPQL& PD-MORL& LP3& \textbf{CoMOGA (Ours)}\\
        \midrule
        Discount factor& 0.99& 0.99& 0.99&0.99\\
        Length of replay buffer& 1000000& 1000000& 100000&100000\\
        Steps per update& 10& 10& 1000&1000\\
        Batch size& 256& 256& 10000&10000\\
 Optimizer& Adam& Adam& Adam&Adam\\
 Policy learning rate& $3 \times 10^{-4}$& $3 \times 10^{-4}$& -&$3 \times 10^{-4}$\\
 Critic learning rate& $3 \times 10^{-4}$& $3 \times 10^{-4}$& $3 \times 10^{-4}$&$3 \times 10^{-4}$\\
 Multiplier learning rate& $1 \times 10^{-5}$& $1 \times 10^{-5}$& $3 \times 10^{-4}$&-\\
 Soft update ratio& 0.005& 0.005& -&-\\
 \texttt{\#} of quantiles to truncate& 2& 2& -&-\\
 \texttt{\#} of HER samples& -& 3& -&-\\
 Angle loss coefficient& -& 10& -&-\\
 Explore and target action noise scale& -& (0.1, 0.2)& -&-\\
 \texttt{\#} of action samples& -& -& 20&-\\
 TD($\lambda$) factor& -& -& 0.97&0.97\\
 \texttt{\#} of target quantiles& -& -& 50&50\\
 \texttt{\#} of preference samples& -& -& 10&10\\
 Local region size& -& -& 0.05&0.05\\
 $g_\mathrm{min}$, $g_\mathrm{max}$& -& -& -&$(0, \infty)$\\
 $\lambda_\mathrm{max}$& -& -& -&$\infty$\\
        $H$ matrix& -& -& Fisher information&Fisher information\\
        \bottomrule
        \end{tabular}
    }
\end{table}

\newpage
\begin{table}[bht]
    \caption{Hyperparameters for MORL tasks.}
    \label{tab: hyperparameter for MORL}
    \centering
    \resizebox{0.95\textwidth}{!}{
        \begin{tabular}{c|cccc}
        \toprule
                & CAPQL& PD-MORL& LP3& \textbf{CoMOGA (Ours)}\\
        \midrule
        Discount factor& 0.99& 0.99& 0.99&0.99\\
        Length of replay buffer& 1000000& 1000000& 100000&1000000\\
        Steps per update& 1& 1& 1000&10\\
        Batch size& 256& 256& 10000&256\\
 Optimizer& Adam& Adam& Adam&Adam\\
 Policy learning rate& $3 \times 10^{-4}$& $3 \times 10^{-4}$& -&$3 \times 10^{-4}$\\
 Critic learning rate& $3 \times 10^{-4}$& $3 \times 10^{-4}$& $3 \times 10^{-4}$&$3 \times 10^{-4}$\\
 Soft update ratio& 0.005& 0.005& -&0.005\\
 \texttt{\#} of HER samples& -& 3& -&-\\
 Angle loss coefficient& -& 10& -&-\\
 Explore and target action noise scale& -& (0.1, 0.2)& -&-\\
 \texttt{\#} of action samples& -& -& 20&-\\
 TD($\lambda$) factor& -& -& 0.97&-\\
 \texttt{\#} of preference samples& -& -& 10&10\\
 Local region size& -& -& 0.14&0.05\\
 $g_\mathrm{min}$, $g_\mathrm{max}$& -& -& -&$(0, \infty)$\\
 $\lambda_\mathrm{max}$& -& -& -&$\infty$\\
        $H$ matrix& -& -& Fisher information&Identity matrix\\
        \bottomrule
        \end{tabular}
    }
\end{table}

\section{Additional Experimental Results}
\label{sec: additional results}

In this section, estimated CP fronts for tasks with two objectives are visualized.
As mentioned in the appendix \ref{sec: performance metric}, the estimated CP front can be obtained by rolling out a trained universal policy with different preferences, and in this section, 20 equally spaced preferences are used.
In Figure \ref{fig: Legged pareto front}, the CP fronts for the legged robot locomotion tasks are visualized, Figure \ref{fig: safety gym pareto front} presents the CP fronts for the Safety Gymnasium tasks, and Figure \ref{fig: MOGym pareto front} presents the CP fronts for the MO-Gymnasium tasks.
The hopper and ant tasks in the MO-Gymnasium are excluded from this visualization as they have three objectives.

\begin{figure}[!ht]
\centering
\begin{subfigure}[b]{0.45\textwidth}
    \centering
    \includegraphics[width=1.0\textwidth]{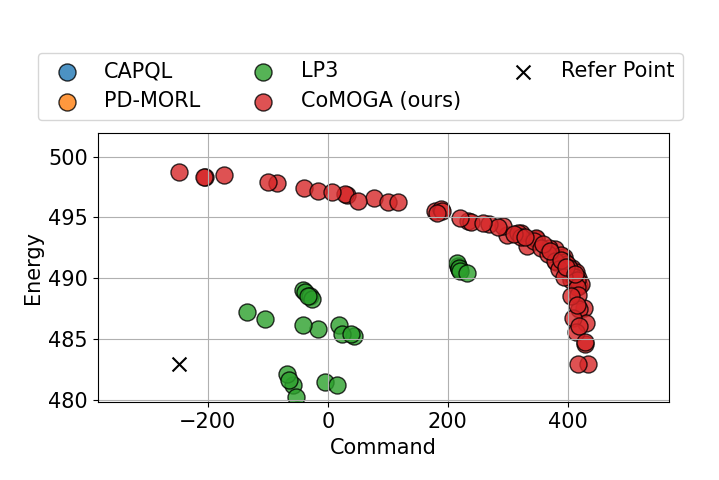}
    \caption{Biped}
\end{subfigure}
\hfill
\begin{subfigure}[b]{0.45\textwidth}
    \centering
    \includegraphics[width=1.0\textwidth]{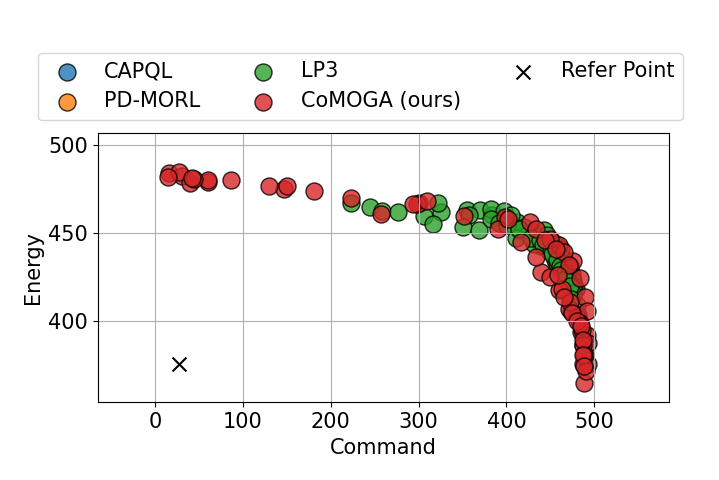}
    \caption{Quadruped}
\end{subfigure}
\caption{Visualization of estimated CP fronts for legged robot locomotion tasks. For each task, we plot the estimated fronts obtained from five random seeds in the same figure.}
\label{fig: Legged pareto front}
\end{figure}

\newpage
\begin{figure}[!ht]
\centering
\begin{subfigure}[b]{0.45\textwidth}
    \centering
    \includegraphics[width=1.0\textwidth]{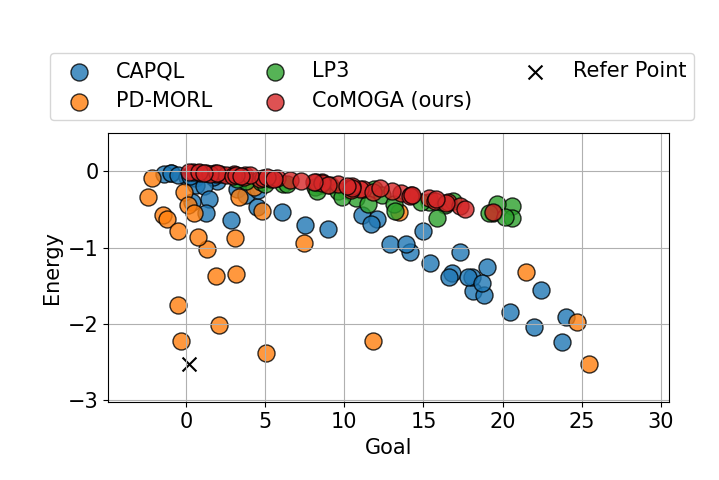}
    \caption{Single-Point-Goal}
\end{subfigure}
\hfill
\begin{subfigure}[b]{0.45\textwidth}
    \centering
    \includegraphics[width=1.0\textwidth]{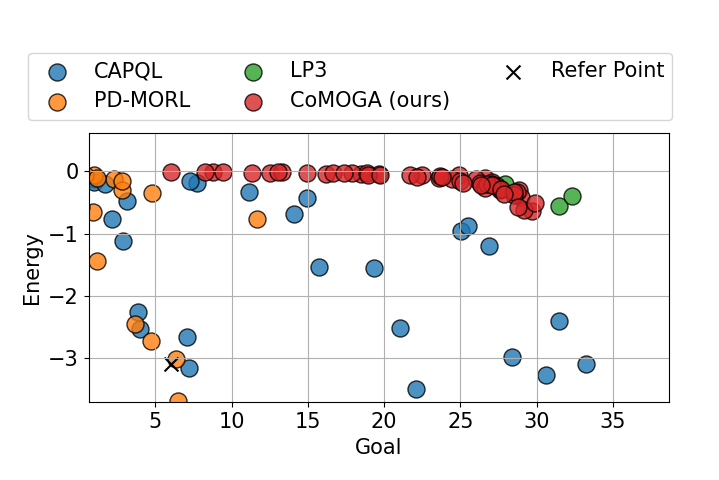}
    \caption{Single-Car-Goal}
\end{subfigure}
\hfill
\begin{subfigure}[b]{0.45\textwidth}
    \centering
    \includegraphics[width=1.0\textwidth]{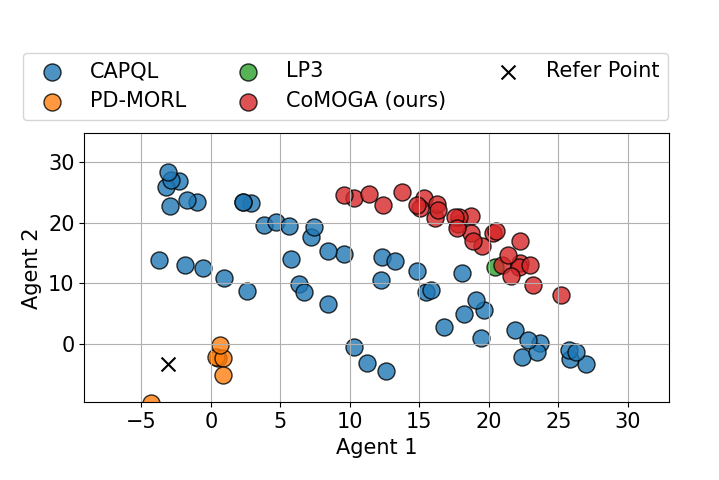}
    \caption{Multi-Point-Goal}
\end{subfigure}
\hfill
\begin{subfigure}[b]{0.45\textwidth}
    \centering
    \includegraphics[width=1.0\textwidth]{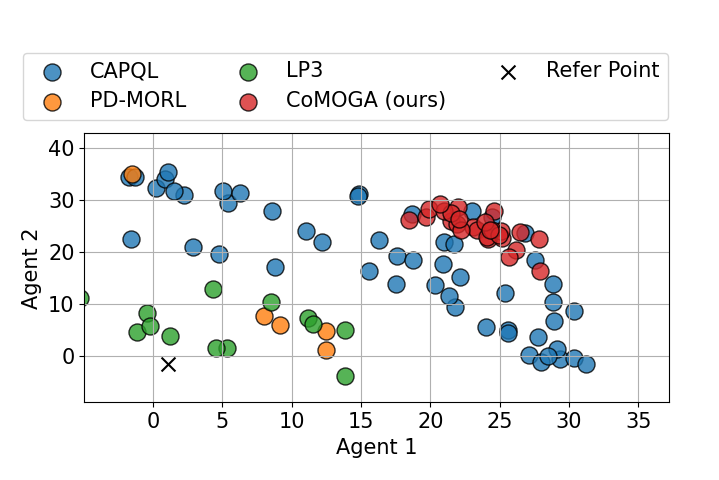}
    \caption{Multi-Car-Goal}
\end{subfigure}
\caption{Visualization of estimated CP fronts for Safety Gymnasium tasks.}
\label{fig: safety gym pareto front}
\end{figure}

\begin{figure}[!ht]
\centering
\begin{subfigure}[b]{0.45\textwidth}
    \centering
    \includegraphics[width=1.0\textwidth]{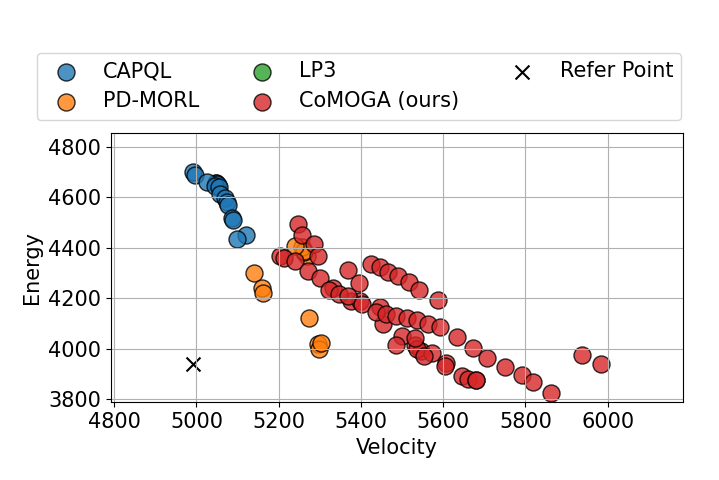}
    \caption{Humanoid}
\end{subfigure}
\hfill
\begin{subfigure}[b]{0.45\textwidth}
    \centering
    \includegraphics[width=1.0\textwidth]{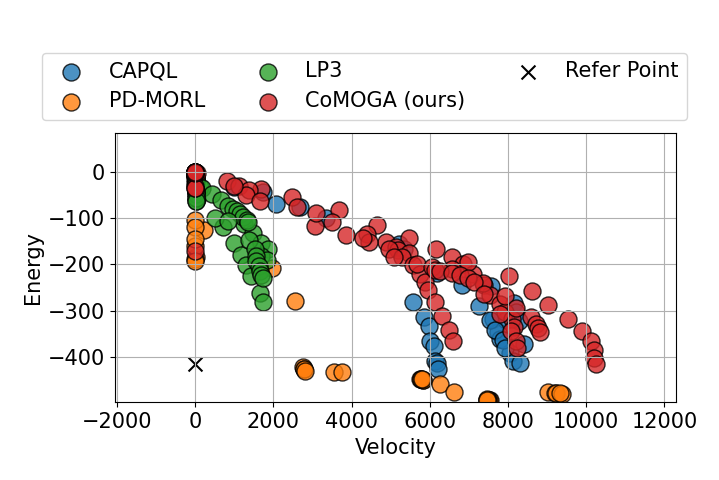}
    \caption{Half-Cheetah}
\end{subfigure}
\hfill
\begin{subfigure}[b]{0.45\textwidth}
    \centering
    \includegraphics[width=1.0\textwidth]{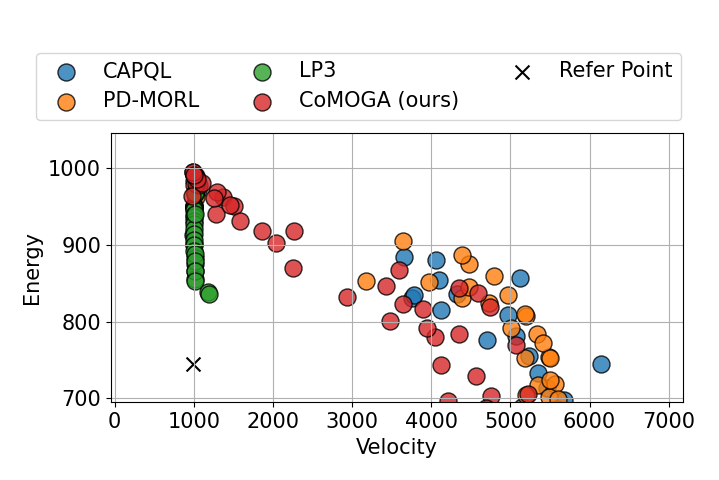}
    \caption{Walker2d}
\end{subfigure}
\hfill
\begin{subfigure}[b]{0.45\textwidth}
    \centering
    \includegraphics[width=1.0\textwidth]{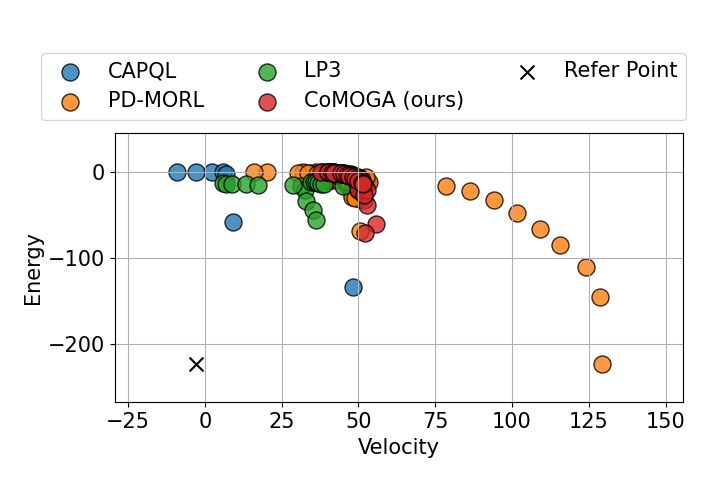}
    \caption{Swimmer}
\end{subfigure}
\caption{Visualization of estimated CP fronts for MO Gymnasium tasks.}
\label{fig: MOGym pareto front}
\end{figure}

\end{document}